%% file: log/log_2026.tex
\documentclass{article}
\usepackage[preprint]{log_2026}			

\usepackage{booktabs}						
\usepackage{multirow}						
\usepackage{amsfonts}						
\usepackage{graphicx}						
\usepackage{duckuments}						

\usepackage{amsmath}
\usepackage{amssymb}
\usepackage{mathtools}
\usepackage{amsthm}
\usepackage{paralist}

\theoremstyle{plain}
\newtheorem{theorem}{Theorem}[section]
\newtheorem{proposition}[theorem]{Proposition}

\theoremstyle{definition}
\newtheorem{definition}[theorem]{Definition}

\theoremstyle{remark}

\newtheorem{example}[theorem]{Example}

\newcommand{\dom}{\mathbf{val}}

\newcommand{\Pred}{\mathsf{rpred}}
\newcommand{\Predn}{\mathsf{npred}}

\newcommand{\ms}[2]{#1{\scriptsize$\pm$#2}}

\usepackage[numbers,compress,sort]{natbib}	

\title[Grables: When Graph Structure Helps Tabular Prediction]{Grables: When Graph Structure Helps Tabular Prediction}

\author[Tamara Cucumides and Floris Geerts]{%
Tamara Cucumides\\
University of Antwerp \\
\email{tamara.cucumidesfaundez@uantwerp.be}\And
Floris Geerts\\
University of Antwerp\\
\email{floris.geerts@uantwerp.be}
}

\begin{document}

\maketitle

\begin{abstract}
Simple row-wise predictors that score each row independently remain competitive with far more expressive architectures on tabular benchmarks, which are i.i.d. by construction. They cannot, however, represent natural targets on transactional, temporal, and relational tables, where labels depend on other rows. We show that row-wise prediction provably rules out targets driven by global counts, overlaps, and relational patterns. To make ``using structure'' precise across architectures, we introduce \emph{grables}: a modular interface that separates how a table is lifted to a graph (constructor) from how predictions are computed on that graph (node predictor), placing gradient-boosted trees, neighbourhood aggregation, transformer-style tabular models, and relational GNNs on a single axis and pinpointing where expressive power comes from. The framework yields a hierarchy that answers when structure is needed: row-local models suffice for extension-invariant targets, lightweight neighbourhood aggregation captures shallow inter-row signals, and explicit graph construction is needed for compositional patterns that fixed aggregations cannot express without exponential feature blow-up. Experiments on synthetic, transactional, and relational data confirm the predicted separations, but also show that benchmark performance is a poor proxy for structural capability: row-local models score well on tasks they provably cannot represent, and a stress-test that suppresses row-local cues is needed to expose the gap. Conversely, access to cross-row context does not imply a model uses it. Feeding explicitly extracted inter-row structure to strong tabular learners yields gains exactly where the hierarchy predicts. The practical recipe is to add structure incrementally, and only where the target is extension-sensitive.
\end{abstract}

\section{Introduction} \label{intro}
\input{log/log-sections/intro-log}

\subsection{Background} \label{sec:prel}
\input{log/log-sections/background-log}

\section{Grables: A unifying interface for tabular and graph learning}
\label{sec:grables}
\input{log/log-sections/grables-log}

\section{Separations: when structure helps, and when message passing fails}
\label{sec:separations}
\label{sec:theory}
\input{log/log-sections/separations-log}

\section{Experiments and results}\label{sec:experiments}
\input{log/log-sections/experiments-log}

\section{Conclusions}
\input{log/log-sections/conclusions-log}

\bibliographystyle{unsrtnat}
\bibliography{references}

\newpage
\appendix
\input{appendix/appendix-A}

\end{document}

%% file: log/log-sections/intro-log.tex

Tabular data underlies a large fraction of real-world machine learning, including applications in healthcare, finance, e-commerce, and the natural sciences. Despite sustained architectural progress, a small class of highly optimized models — notably gradient-boosted decision trees \cite{xgboost, lightgbm_neurips, catboost} and carefully tuned multilayer perceptrons \cite{gorishniy2025tabm, realmlp_neurips} — remains hard to beat. A common and largely implicit assumption shared by these methods is \emph{row-locality}: predictions are made independently for each row, using only the attributes of that row, and are invariant to the presence, absence, or duplication of other rows in the table. This assumption is built into standard i.i.d.\ benchmark constructions \cite{erickson2025tabarena}, which never require a model to look beyond the row and therefore offer little room for more expressive alternatives to pay off.

At the same time, many practically important tabular datasets violate this assumption. In transactional, temporal, and entity-centric tables, prediction targets often depend on relationships between rows, such as shared values, global counts, overlaps, or co-occurrence patterns. While such dependencies are routinely modelled in specialized settings, they remain underexplored in general-purpose tabular learning. Recent approaches, including graph-based models \cite{gnn4td} and transformer-style tabular foundation models that relax row-locality via cross-row interactions \cite{tabpfn_nature,qu2025tabicl,ma2025tabdpt}, report gains that are difficult to interpret: it is often unclear whether they reflect principled representational advantages or increased capacity or optimization effects.

This motivates a fundamental question: when does modelling inter-row structure provide a genuine advantage over row-local prediction, and what precisely enables that advantage? \emph{Our answer is that this advantage is fundamentally representational: certain natural, extension-sensitive targets are inaccessible to any row-local predictor, regardless of model capacity.} To make this precise, we introduce \emph{grables}, a modular interface that separates how a table is lifted to a relational structure from how predictions are computed on that structure. A grable consists of a \emph{graph constructor}, which maps a table to a graph while preserving a canonical correspondence between rows and nodes, and a \emph{node predictor}, which operates on the resulting graph. This abstraction makes explicit where structural inductive bias enters and provides a common language to compare row-wise tabular models, message-passing neural networks, and transformer-style tabular methods whose structure is implicit in attention.

Building on this interface, we formalize \emph{grabular expressibility}, an expressiveness framework that characterizes which row-level prediction functions can be realized by a given family of constructors and predictors. Unlike standard neural network analyses \cite{Bar+2020, morris-WL-go-neural}, which fix the input graph, our framework quantifies over the families of graph views induced from the same underlying table. Using this lens, we show that row-local predictors are exactly those that are invariant under table extensions and therefore incapable of representing natural inter-row targets driven by counts, overlaps, or shared witnesses. In contrast, message-passing models that operate on simple, schema-driven constructions such as incidence graphs can represent such targets, but only up to well-defined limits that can be characterized using logical tools.

We complement this analysis with controlled experiments on synthetic and real transactional data that isolate extension-sensitive dependencies. These experiments empirically validate the predicted separations and show that apparent gains by row-local models—including strong tabular baselines and tabular foundation models—can arise from row-local cues rather than genuine relational reasoning. Finally, we show that explicitly exposing inter-row structure and combining it with strong tabular representations yields consistent improvements on relational data, which we examine in detail through a case study, while leaving performance unchanged on i.i.d.\ benchmarks curated to be extension-invariant. Gains therefore appear where the hierarchy predicts they should and nowhere else, indicating that tabular and graph-based representations are complementary rather than competing.

See Appendix~\ref{sec:related} for a detailed discussion of related work.

%% file: log/log-sections/background-log.tex

We recall the basic objects used throughout: tables, graphs, MPNNs, and their
logic-based expressiveness. For sets $X,Y$,
$Y^X\coloneqq\{f\mid f:X\to Y\}$ is the set of functions from $X$ to $Y$.

\textbf{Tables and row predictors.}
Let $\mathcal{C}$ be a universe of \emph{column names} and $\dom$ a \emph{value domain}; for simplicity, every column $c\in\mathcal{C}$ takes values in $\dom$. A \emph{schema} is a finite set $C\subseteq\mathcal{C}$, a \emph{$C$-row} is a function $r:C\to\dom$ (equivalently, an element of $\dom^C$), and a \emph{$C$-table} is a finite set $T\subseteq\dom^C$. When $C$ is clear we simply say \emph{row} and \emph{table}. A \emph{row predictor} is a function $\Pred$ that maps each pair $(C,T)$ to a \emph{per-row classifier} $
\Pred_{C,T}:T\to\{0,1\}:r\mapsto \Pred_{C,T}(r)$ \footnote{The dependence on $T$ is important: the prediction for a row $r$ may or may not depend on the other rows present in the table. This is precisely what separates row-local predictors from predictors that exploit inter-row structure}.

\textbf{Graphs and node predictors.}
A \emph{graph} is a tuple $G=(V,E_1,\ldots,E_m,\rho)$, where $V$ is a finite set of nodes, each $E_i\subseteq V\times V$ is a symmetric edge relation (encoding undirected edges), and $\rho$ assigns \emph{row-valued features}: each node $v\in V$ carries a finite \emph{local schema} $C(v)\subseteq\mathcal{C}$ and a feature row $\rho(v)\in\dom^{C(v)}$, so nodes may have different local schemas. We let $N_i(v)$ denote vertices adjacent to $v$ along $E_i$-edges. When we write $(v,w)$ for an edge, we mean the undirected edge between $v$ and $w$. A \emph{node predictor} $\Predn$ assigns to each $G$ a (possibly partial) node classifier $\Predn_G:V\to\{0,1\}$. When needed, we further enrich $G$ with edge labels or weights via a map $\eta:\bigcup_i E_i\to\Gamma$; in each case, we will state explicitly whether this additional structure is part of $G$.

\textbf{MPNNs and their logical characterization.}
A \emph{$k$-layer (heterogeneous) Message-Passing Neural Network} \citep{Gil+2017,Schlichtkrull+2018} computes node embeddings $\mathbf{h}_v^{(\ell)}\in\mathbb{R}^d$ for $\ell=0,\ldots,k$, initialized from $\rho(v)$ and updated by permutation-invariant aggregation of messages from neighbours (per edge type); a learnable readout followed by thresholding yields a node
classifier $\Predn_G^{\mathrm{MPNN}}$; full details are given in
Appendix~\ref{appsec:background}. Since a $k$-layer MPNN can only use
information from a node's $k$-hop neighbourhood, its expressiveness has a clean
logical description. First-order logic (FO) defines unary graph properties,
while graded modal logic (GML) captures neighbourhood reasoning with threshold
counts. We write $\mathrm{GML}^{(\le k)}$ for the fragment of modal depth at
most $k$.

\begin{theorem}[\citet{Bar+2020}]
\label{thm:fo-mpnn-gml}
Fix $k\in\mathbb N$ and a unary FO formula $\varphi(x)$ over a graph schema. The following are equivalent: (i) $\varphi$ is realized by some $k$-layer MPNN node classifier, i.e., $\Predn^{\mathrm{MPNN}}_G(v)=1 \iff G\models \varphi(v)$ for all $G,v$; (ii) $\varphi$ is equivalent to some $\psi(x)\in \text{\normalfont GML}^{(\le k)}$.
\end{theorem}

For instance, on the incidence graph introduced in Section~\ref{sec:separations},
the property ``a row shares a value with at least $N$ other rows'' is definable
in $\mathrm{GML}^{(\le 2)}$, and is therefore realizable by a 2-layer MPNN. See
Appendix~\ref{appsec:background} for the formal syntax and semantics.

%% file: log/log-sections/grables-log.tex
Tabular prediction sits at an awkward crossroads.
Classical tabular methods (e.g., gradient-boosted trees) operate \emph{directly} on rows and columns and typically expose no explicit relational structure beyond what is manually engineered.
At the other extreme, relational and graph-based models \emph{start} from an explicit graph and learn over its neighbourhoods.
Recent transformer-style tabular models, retrieval- and attention-based predictors, and graph foundation models further blur the picture: they may capture rich interactions between cells, rows, or neighbours, but the induced ``structure'' often lives implicitly in attention, retrieval, or tokenization, and is therefore hard to compare across architectures.
This raises a basic question: \emph{what does it even mean to compare the expressive power of methods that (i) see only rows, (ii) see an explicit graph, or (iii) induce structure implicitly?}

We address this by making the missing object explicit.
We separate (a) \emph{how a table is turned into a relational structure} from (b) \emph{how predictions are computed on that structure}.
The key abstraction is a \emph{graph constructor} $\gamma$ that equips each table with a graph representation while preserving a canonical correspondence between rows and nodes.

\subsection{Grables}
Let $C$ be a schema and $T$ a $C$-table.
A \emph{graph constructor} $\gamma$ maps each pair $(C,T)$ to a \emph{graph}
$G_{C,T}^\gamma=(V,E_1,\ldots,E_m,\rho)$.
\emph{Crucially}, $G_{C,T}^\gamma$ contains one distinguished node for each row $r\in T$; we denote it by $v_r$ and call it the \emph{row node}.
Row nodes store the row itself as features: $\rho(v_r)=r\in\dom^C$.
Beyond row nodes, $\gamma$ may introduce auxiliary nodes (e.g., values, column identifiers, join keys), edge relations, and additional features.
We call any graph obtained from a table in this way a \emph{grable}.

We give two natural examples. The \emph{trivial grable} simply encodes a table as an edgeless graph with one node per row, while the \emph{incidence grable} (used by \citet{Gro+2020} to study the expressive power of MPNNs on relational data) introduces value nodes and connects each row to its column--value pairs.

\begin{example}[Trivial grable]\label{ex:trivial-grable}
For any schema $C$ and $C$-table $T$, the \emph{trivial grable} of $(C,T)$ is the graph
$G^{\mathrm{triv}}_{C,T}=(V,E,\rho)$ where $V=\{v_r \mid r\in T\}$, $E=\emptyset$, and $\rho(v_r)=r$ for all $r\in T$.
We write $\gamma_{\mathrm{triv}}$ for the corresponding constructor.\hfill $\blacktriangleleft$
\end{example}

\begin{example}[Incidence grable]\label{ex:incidence-grable}
For any schema $C$ and $C$-table $T$, write $C=\{c_1,\dots,c_m\}$.
The \emph{incidence grable} of $(C,T)$ is the graph
$G^{\mathrm{inc}}_{C,T}=\bigl(V,\;E_1,\dots,E_m,\;\rho\bigr)$
defined as follows.
\begin{compactitem}
    \item $V=V_{\mathrm{row}}\cup V_{\mathrm{val}}$, where $V_{\mathrm{row}}=\{v_r \mid r\in T\}$ are the row nodes and
$
    V_{\mathrm{val}}=\{u_{i,a}\mid i\in[m],\ a\in \dom,\ \exists r\in T: r[c_i]=a\}
 $   contains one \emph{value node} for each column--value pair that occurs in $T$.
    \item For each $i\in[m]$, the edge relation $E_i\subseteq V^2$ is given by
    $(v_r,u_{i,a})\in E_i$ iff $r[c_i]=a$ (optionally also adding $(u_{i,a},v_r)$ to make edges undirected).
    \item Row nodes store full rows, $\rho(v_r)=r\in\dom^C$; value nodes may store their value (and optionally the column name), e.g.\ $\rho(u_{i,a})(c_i)=a$.
\end{compactitem}
We write $\gamma_{\mathrm{inc}}$ for the corresponding constructor.\hfill $\blacktriangleleft$
\end{example}

We can also formalize grables corresponding to a wide range of existing methods by making their tokenization, attention, or retrieval patterns explicit as graph constructors.
This includes transformer-style tabular models such as TabPFN, CARTE and TARTE \citep{tabpfn_nature,kim2024carte, kim2025table}; retrieval and attention-based predictors such as SAINT, NPT, TabR, and ModernNCA \citep{saint, kossen-npt, tabr, modern-nca}; and graph foundation models such as GraphPFN, G2T-FM, and TabPFN-GN \citep{graphpfn, g2t, choi2025can}.
We also consider Neighbourhood Feature Aggregation \citep{bazhenov2025GraphLand}, which is a \emph{derived} constructor: it requires a starting grable to define $1$-hop neighbourhoods, but its output is a \emph{trivial grable} so that downstream learners can be purely tabular. We refer for details of example constructors to Appendix~\ref{appsec:grables}; some use an edge function $\eta$ to encode, for example, attention scores, retrieval weights, distances, or relation-specific edge labels. 

\subsection{Grabular expressiveness}
\label{sec:grabular-expressiveness}

Grables let us put tabular, relational, and transformer-style models on the same axis.
The crucial point is that, for tables, there is \emph{no single given graph}:
the model must first \emph{choose a view} of the table as a graph structure, and only then compute predictions.
This yields a two-part inductive bias:

\begin{compactitem}
    \item a \emph{constructor} $\gamma$ that decides \emph{what structure the model is allowed to use} (which nodes/edges/features), and
    \item a \emph{node predictor} $\Predn$ that decides \emph{how predictions are computed} on that structure.
\end{compactitem}

This separation is exactly what is missing when comparing
(i) classical tabular methods that ``see only rows'',
(ii) graph (and also relational) models that start from an explicit neighbourhood structure,
and (iii) transformer-style tabular models whose structure is implicit in attention.
We propose an expressiveness notion that makes these differences transparent: it attributes power either to a \emph{stronger graph view} or to a \emph{stronger predictor on a fixed view}.

Let $\Gamma$ be a set of admissible graph constructors and let $\mathcal P$ be a class of node predictors.

\begin{definition}[Grabular expressibility]
\label{def:expressibility}
A row predictor $\Pred$ is \emph{$(\Gamma,\mathcal P)$-expressible} if there exist
a constructor $\gamma\in\Gamma$ and a node predictor $\Predn\in\mathcal P$ such that for \emph{all} schemas $C$,
all $C$-tables $T$, and all rows $r\in T$,
\[
\Pred_{C,T}(r)=\Predn_{G^\gamma_{C,T}}(v_r),
\]
where $v_r$ denotes the distinguished row node corresponding to $r$ in the grable $G^\gamma_{C,T}$.
\end{definition}

Equivalently, $\Pred$ is expressible if there is a \emph{single, uniform} way to turn every table into a graph
(using some $\gamma\in\Gamma$) so that the desired row-level decision reduces to a node-level decision on row nodes.
We write
$
\mathsf{Expr}(\Gamma,\mathcal P)
:=\{\Pred \mid \Pred \text{ is } (\Gamma,\mathcal P)\text{-expressible}\}$
for the resulting expressiveness class. Importantly, we can now compare models by set inclusion:
$(\Gamma,\mathcal P)$ is (weakly) at least as expressive as $(\Gamma',\mathcal P')$ if
$\mathsf{Expr}(\Gamma',\mathcal P')\subseteq \mathsf{Expr}(\Gamma,\mathcal P)$, etc.

\textbf{Why this differs from standard graph expressiveness.}
Expressiveness results for MPNNs and graph logics typically \emph{fix the input graph} $G$ and study what a predictor can compute \emph{on that graph}.
In contrast, Definition~\ref{def:expressibility} quantifies over a \emph{family of graph views of the same underlying object}, that is, a table.\looseness=-1

\textbf{Why restrictions matter.}
Without restrictions on $\Gamma$ and $\mathcal P$, the definition becomes vacuous:
with $\gamma_{\mathrm{triv}}$ (isolated row nodes with $\rho(v_r)=r$) and an unconstrained $\Predn$,
one can hard-code any $\Pred_{C,T}$. We
only obtain meaningful insights by fixing constructor families $\Gamma$ (e.g.\ incidence, CARTE/TARTE-style graphlets, attention-induced graphs)
and natural predictor classes $\mathcal P$ such as those expressible by MPNNs.
In Section~\ref{sec:separations} we fix $\mathcal P$ to be $k$-layer MPNNs and study how the choice of $\Gamma$ shapes expressibility, deriving separations between row-local, incidence-based, and richer constructors.

%% file: log/log-sections/separations-log.tex
In what follows, we fix the predictor side to be $k$-layer MPNNs, i.e., $\mathcal P=\mathcal P_{\mathrm{MPNN}}^{(k)}$,
and ask how the choice of constructor family $\Gamma$ shapes $\mathsf{Expr}(\Gamma,\mathcal P_{\mathrm{MPNN}}^{(k)})$. Proof details are in Appendix~\ref{appsec:results}.

\subsection{Baseline: on the trivial grable, MPNNs are row-wise}
\label{sec:trivial-grable-rowlocal}

The trivial constructor $\gamma_{\mathrm{triv}}$ creates  the edgeless graph
$G^{\mathrm{triv}}_{C,T}$ with one node $v_r$ per row $r\in T$ and node features $\rho(v_r)=r$.
With no edges, message passing cannot transmit information between rows.

\begin{proposition}[MPNNs on $\gamma_{\mathrm{triv}}$ are row-wise]
\label{lem:mpnn-triv-rowwise}
Fix $k\in\mathbb N$. For any $k$-layer MPNN node classifier on $G^{\mathrm{triv}}_{C,T}$, there exists a function
$f:\dom^C\to\{0,1\}$ (determined by the learned parameters) such that for all $r\in T$,
\[
\Predn^{\mathrm{MPNN}}_{G^{\mathrm{triv}}_{C,T}}(v_r)=f(r).
\]
In particular, the prediction for $r$ depends only on the column values of $r$ and not on the rest of $T$.
\end{proposition}

This motivates the standard tabular baseline: a row predictor is \emph{extension-invariant} (or \emph{row-local}) if its output on a row does not change when other rows are added or removed from the table.

\begin{definition}[Row-locality / extension invariance]
A row predictor $\Pred$ is \emph{row-local} if for all $(C,T)$ and $r\in T$,
$
\Pred_{C,T}(r)=\Pred_{C,\{r\}}(r)$,
equivalently, if for every schema $C$ there exists a function
$f_C:\dom^C\to\{0,1\}$ such that
$\Pred_{C,T}(r)=f_C(r)$ for all $T$ and $r\in T$.
\end{definition}

A row predictor is \emph{extension-sensitive} if it is not row-local: there exist tables $T\subsetneq T'$ and a row $r\in T$ such that $\Pred_{C,T}(r)\neq \Pred_{C,T'}(r)$. Extension-sensitive targets are precisely those whose values depend on relations among rows, and are the central object of study in the rest of the paper.

Classical tabular models such as random forests and MLPs are row-local at inference: once trained, they implement a map $f(r)$ and do not depend on the rest of $T$.
Also, CARTE and TARTE do \emph{not} introduce cross-row connectivity:
$\gamma_{\mathrm{CARTE}}$ builds a disjoint union of per-row star graphlets, and $\gamma_{\mathrm{TARTE}}$ connects only tokens within the same row (cf.~Appendix~\ref{appsec:grables}).
In contrast, any $(\Gamma,\mathcal P_{\mathrm{MPNN}}^{(k)})$ with $\Gamma$ containing constructors that create cross-row edges admits non-row-local predictors. 
For example, on the incidence grable, rows can exchange information through shared value nodes; and in TabPFN, column-wise attention enables cross-row information flow.



\textbf{Takeaway.}
$\gamma_{\mathrm{triv}}$ is the exact baseline for row-wise tabular learning: with no edges, graded modal operators are vacuous and any depth-$k$ GML row property reduces to a Boolean combination of unary predicates on the row's own features, so every $k$-layer MPNN on this constructor collapses to an independent per-row classifier.

\subsection{Canonical inter-row structure: the incidence grable}
\label{sec:incidence-separations}

To expose inter-row dependencies, we use the schema-driven incidence constructor $\gamma_{\mathrm{inc}}$.
It creates a row node $v_r$ for each $r\in T$ and a value node $u_{i,a}$ for each occurring column--value pair $(c_i,a)$,
with a typed edge $(v_r,u_{i,a})\in E_i$ iff $r[c_i]=a$.
This bipartite encoding makes ``rows share values'' relationships explicit to message passing. For the logical expressiveness results, as required by Theorem~\ref{thm:fo-mpnn-gml}: Actual domain
values are not exposed as value-specific unary predicates or duplicated
row-node features; otherwise the signature would become infinite or
table-dependent.

\textbf{Four minimal inter-row targets.}
We use four intentionally simple per-row predictors, illustrated in Figure~\ref{fig:logical-tasks-incidence}:
\textsc{Unique} (a value appears only in $r$),
\textsc{Count}$_N$ (at least $N$ other rows share a designated value with $r$, with \textsc{Count}($=$) as the exactly $N$ case),
\textsc{Double} (\emph{another} row shares a value with $r$ and satisfies a local condition),
and \textsc{Diamond} (there is \emph{another} row sharing two values with $r$).

These targets are not arbitrary. \textsc{Unique}, \textsc{Count} and \textsc{Double} are basic patterns of two-variable logic with counting quantifiers, and \textsc{Diamond} is the simplest pattern beyond it. Each also has a direct counterpart in fraud detection: \textsc{Unique} corresponds to a customer's first transaction at a merchant, \textsc{Count} to a card used more than $N$ times in a time window, \textsc{Double} to another transaction on the same card coming from a flagged IP, and \textsc{Diamond} to another transaction sharing both merchant and IP. All four genuinely require inter-row information and therefore lie outside the row-local baseline captured by $\gamma_{\mathrm{triv}}$. This makes them an instrument for the theory rather than a proposed benchmark: they are the setting in which a representability claim can be falsified.

\begin{figure}[t]
    \centering
\includegraphics[width=\linewidth]{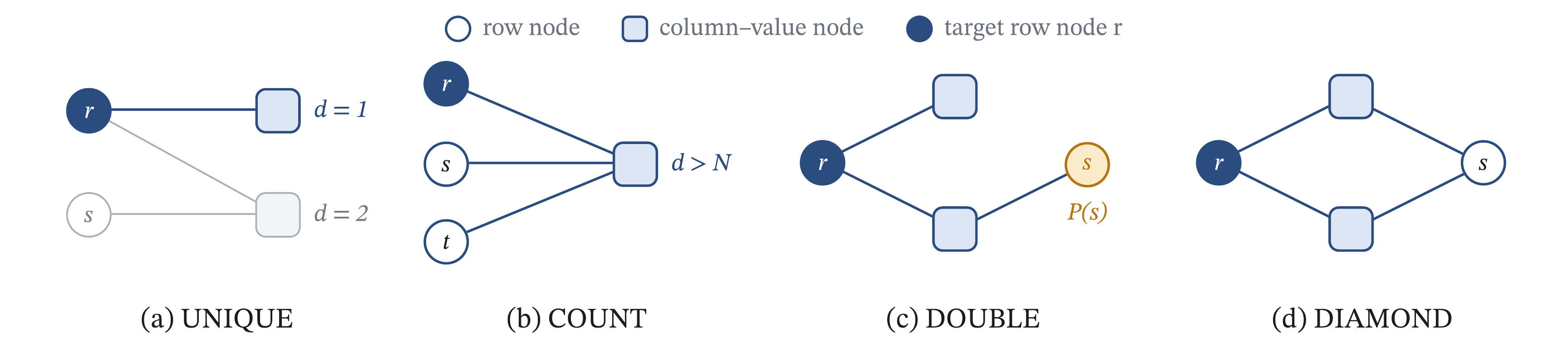}
    \caption{Incidence-grable patterns for the four logical tasks. Row nodes (circles) connect to column--value nodes (squares) via typed edges; the target row $r$, whose label is predicted, is filled. We write $d$ for the degree of a column--value node, i.e.\ the number of rows carrying that value. \textbf{(a)}~\textsc{Unique}$(r)$: $r$ is adjacent to a value node of degree $1$; the second value node, of degree $2$, is not a witness. \textbf{(b)}~\textsc{Count}$_N(r)$: $r$ is adjacent to a value node of degree greater than $N$, i.e.\ at least $N$ other rows share that value. \textbf{(c)}~\textsc{Double}$(r)$: a length-3 pattern $r - u_{i,v} - s - u_{j,w}$ whose witness row $s$ satisfies the local condition $P$. \textbf{(d)}~\textsc{Diamond}$(r)$: a single row $s$ shares two distinct column--value nodes with $r$. In all four cases the label of $r$ depends on rows other than $r$.}
    \label{fig:logical-tasks-incidence}
\end{figure}

\begin{proposition}[Incidence separates row-locality]
\label{prop:inc-not-rowlocal}
Each predictor in Figure~\ref{fig:logical-tasks-incidence} is not row-local.
Equivalently, none lies in
$\mathsf{Expr}(\{\gamma_{\mathrm{triv}}\},\mathcal P_{\mathrm{MPNN}}^{(k)})$
for any $k$.
\end{proposition}

On $\gamma_{\mathrm{inc}}$, however, the first three tasks fall squarely within graded modal logic: they can be written as shallow $\mathrm{GML}$ formulas over the incidence signature (modal depth $2$ for \textsc{Unique}/\textsc{Count}$_N$, and depth $3$ for \textsc{Double}).
By Theorem~\ref{thm:fo-mpnn-gml} (FO$\cap$MPNN $\Leftrightarrow$ GML), this immediately implies realizability by constant-depth MPNNs.

\begin{proposition}
\label{prop:inc-mpnn-modal}
Let $\mathcal P_{\mathrm{MPNN}}^{(k)}$ denote $k$-layer MPNNs. For the tasks in Figure~\ref{fig:logical-tasks-incidence}(a--c),
\begin{compactitem}
\item $\textsc{Unique},\ \textsc{Count}_N\ \in\ \mathsf{Expr}(\{\gamma_{\mathrm{inc}}\},\mathcal P_{\mathrm{MPNN}}^{(2)})$,
\item $\textsc{Double}\ \in\ \mathsf{Expr}(\{\gamma_{\mathrm{inc}}\},\mathcal P_{\mathrm{MPNN}}^{(3)})$.
\end{compactitem}
\end{proposition}

The remaining task, \textsc{Diamond}, is the first point where one must \emph{bind a shared witness}: the same other row must simultaneously witness two overlaps with $r$ (Figure~\ref{fig:logical-tasks-incidence}(d)).
This is FO-definable on the incidence graph, but not in GML, because GML cannot bind the same witness row across two distinct shared-value tests. By Theorem~\ref{thm:fo-mpnn-gml} (FO$\cap$MPNN $\Leftrightarrow$ GML), no bounded-depth MPNN can express it.

\begin{proposition}[Diamond is beyond  MPNNs on incidence grable]
\label{prop:inc-fo-strict}
For every $k\in\mathbb N$,
\[
\textsc{Diamond}\notin \mathsf{Expr}(\{\gamma_{\mathrm{inc}}\},\mathcal P_{\mathrm{MPNN}}^{(k)}).
\]
\end{proposition}

\textbf{Takeaway.}
Incidence gives a principled ``relational lift'' of tables: it strictly enlarges what is learnable beyond row-local tabular baselines,
yet logic still certifies a clean expressiveness gap for message passing. One can, however, also see this as a limitation of the \emph{constructor} rather than of message passing. In Appendix~\ref{appsec:results} we give a simple \emph{extension} of the incidence grable under which \textsc{Diamond} becomes definable in shallow GML and therefore MPNN-expressible.

\subsection{MPNNs on NFA: graph power compiled into feature width}
\label{sec:nfa-mpnn}
Neighbourhood Feature Aggregation (NFA) is \emph{derived} from a constructor $\gamma$: it first builds $G^\gamma_{C,T}$ to define $1$-hop neighbourhoods, computes a fixed menu of local aggregates around each row node, and then \emph{discards the graph} by outputting the trivial grable $G^{\mathrm{NFA}}_{C,T}=(V_{\mathrm{row}},\emptyset,\rho^{\mathrm{NFA}})$ (Appendix~\ref{app:nfa_grable}). \looseness=-1

A predictor applied \emph{after} NFA is graph-agnostic and can access inter-row information only through the compiled features $\rho^{\mathrm{NFA}}$ (e.g., counts and statistics such as $\min/\max/\mathrm{mean}$); a \texttt{group\_by} count or aggregate is precisely such a channel, so NFA is the counterpart of feature engineering over groups. This makes NFA effective for shallow relational signals, but deeper patterns must be compiled by iterating augmentations or adding channels: on an incidence-style base graph with $k$ typed attribute edges, the number of distinct typed $L$-step interaction patterns grows as $k^L$, so preserving the information available after $L$ rounds of typed message passing can require exponentially many aggregate channels. This limits the effectiveness of NFA for deep relational patterns.

\subsection{Inference protocols: which information is available at inference time?}
\label{sec:inference-protocols}
Expressiveness is a property of the constructor--predictor pair, but a model can only use it if it is given the rows it needs. The inference protocol---which table context the constructor receives at evaluation time---can lower this ceiling but never raise it. If the constructor is handed only the row being predicted on, the pair is row-local, no matter how expressive it is.

All protocols take the form $\Pred_{C,R\cup\{r\}}(r)$ and differ only in the reference set $R$:
\begin{compactitem}
    \item $R=\emptyset$: the constructor builds a grable on $\{r\}$ alone. This is the protocol of classical row-local learners, and it forces row-locality.
    \item $R$ is the training table, or training plus validation: a fixed, labeled reference set, with only $r$ scored. TabR, ModernNCA~\citep{tabr,modern-nca} and TabPFN~\citep{tabpfn_nature} use this protocol.
    \item $R$ is any other set of rows---the rest of the evaluation table, or an entirely new table---scored together with $r$.
\end{compactitem}
This is the tabular counterpart of the transductive--inductive distinction from graph learning. In tabular learning it is left implicit, because under row-locality it makes no difference. It matters here: extension-invariant targets are measurable under any protocol, extension-sensitive ones only when $R\neq\emptyset$. The choice of $R$ therefore decides whether an evaluation is informative at all (Appendix~\ref{app:inference_protocol}).

%% file: log/log-sections/experiments-log.tex
Our experiments are diagnostic, not competitive: we do not aim for state-of-the-art performance but use the framework to validate and stress-test the theory. Within our framework, classical tabular learners, GNNs on incidence and relational entity graphs, and tabular transformers are all instances of the same constructor--predictor template, differing only in the $(\Gamma,\mathcal P)$ pair they instantiate. We use them as controlled interventions in this design space, organized around three questions:
\textbf{(Q1)}~\emph{Can the separation of row-local vs.\ message-passing be observed empirically?} We use synthetic data and tasks that isolate extension-sensitive dependencies, directly testing the theoretical distinctions of Section~\ref{sec:theory}. A stress-test set that removes row-local cues checks whether the observed scores reflect inter-row structure.
\textbf{(Q2)}~\emph{Does the same separation arise on real tabular data with naturally occurring inter-row dependencies?} We move to a real transactional dataset while retaining tasks that explicitly depend on relationships between rows.
\textbf{(Q3)}~\emph{Are tabular and graph-based representations complementary in practice?} We walk through a relational case study where row-local models are unexpectedly strong, to illustrate how the framework guides model choice.

\subsection{Experimental setup}
Code available at: \url{https://github.com/TamaraCucumides/Grables}.

\textbf{Datasets.} We focus on three settings, each aimed at one of our research questions: a synthetic transactions dataset with four categorical features; a real-world retail transactions dataset where inter-row dependencies arise naturally~\citep{retailDataset_UCI}; and the clinical trials dataset \texttt{relbench-trial} from the RelBench benchmark~\citep{relbench_neurips}, used in our case study. Statistics and details are in Appendix~\ref{app:datasets}. These settings are chosen to controllably stress representational limitations: generic tabular benchmarks are curated under i.i.d.\ assumptions that align with row-local learning, while graph benchmarks presuppose explicit structure favoring structural models. Supplementary results on TabArena~\citep{erickson2025tabarena} (Appendix~\ref{app:tabarena}) confirm that gains from exposing inter-row structure are modest in the i.i.d.\ regime, consistent with our theory.

\textbf{Tasks and leakage control.}
We instantiate the logical tasks introduced in Section~\ref{sec:separations} on the synthetic and transaction datasets.
To prevent information leakage, tasks are generated \emph{independently within each split}, using disjoint train/validation/test partitions. When graph representations are used, a separate graph is constructed per split and all learning is performed in an inductive setting. For the \texttt{relbench-trial} dataset, we follow the standard RelBench protocol and predefined splits. Full task-generation details are provided in Appendix~\ref{app:task-generation}.

\textbf{Stress-test set.}
On the original test splits, labels of extension-sensitive tasks can correlate with row-local features, so a model can score well on a task it cannot represent. For the synthetic data we therefore also construct, for each task, a stress-test set that removes these correlations: values unique in training occur repeatedly, values frequent in training fall below the \textsc{Count} threshold, the \textsc{Double} label cannot be predicted from either condition alone, and frequent single values no longer predict a repeated \textsc{Diamond} pair. On this set, a predictor that depends only on the row cannot score above chance. Models are trained on the original training split and evaluated on the stress-test set without retraining, using the thresholds selected on the original validation split; as for the other splits, a separate graph is built on the stress-test set. Construction details are in Appendix~\ref{app:stress-set}.

\textbf{Models.}
We instantiate several $(\Gamma,\mathcal P)$ pairs. Row-local constructors are paired with LightGBM~\citep{lightgbm_neurips} and RealMLP~\citep{realmlp_neurips}, the ModernNCA~\citep{modern-nca} constructor connects each query row to all training rows via a learned embedding and predicts via soft nearest-neighbor aggregation. TabPFN~\citep{tabpfn_nature} instantiates the $\gamma_{\text{TabPFN}}$ constructor (Appendix~\ref{app:tabpfn_grable}) with its native row/column attention predictor. Both ModernNCA and TabPFN use the training table as their reference for prediction. Incidence-grable constructors are paired with heterogeneous message passing (HeteroSage~\citep{GraphSage}). The relational case study includes models operating on the relational entity graph: RDL, RGT, and KumoRFM~\citep{relbench_neurips, dwivedi2025relationalgraphtransformer, fey2025kumorfm}. Details are in Appendices~\ref{app:tabular_details}, \ref{app:tabpfn-settings}, and~\ref{app:gnn_details}.

\textbf{Evaluation metrics.}
Across all experiments, performance is measured using ROC--AUC. Because the logical tasks naturally induce class imbalance, we additionally report F1 scores where appropriate; for all models, classification thresholds are selected on the validation set to maximise F1 score.

\subsection{Results on logical tasks}
\input{neurips/neurips-tables-figures/table-synthetic-transactional}

Table~\ref{tab:logical-tasks} addresses \textbf{(Q1)} on the synthetic data and \textbf{(Q2)} on the real retail transactions. The results align with the theory: on \textsc{Unique} and \textsc{Count}, the GNN on the incidence grable achieves near-perfect scores while tabular models and TabPFN lag substantially; on \textsc{Double} the GNN remains competitive; and on \textsc{Diamond}, which is provably beyond bounded-depth MPNNs on the incidence grable (Proposition~\ref{prop:inc-fo-strict}), the GNN's advantage disappears. The same pattern holds on retail data, with two practical caveats: the gap on \textsc{Double} and \textsc{Diamond} narrows under feature noise, and the GNN's performance on \textsc{Count}$(=)$ degrades in the presence of heterogeneous task-irrelevant features, recovering when this noise is reduced (Appendix~\ref{app:transactions-remove-noise}). The equality variant of \textsc{Count} is also harder in general due to its finer decision boundary.

\emph{The framework explains why TabPFN and ModernNCA do not capture inter-row signal on extension-sensitive targets, despite their non-trivial constructors.} The framework predicts these constructors are expressive enough for the tasks; what limits them is their deployed inference protocol, which withholds the row context they would need to use that expressiveness.

\subsection{Stress test: separating representation from shortcuts}\label{sec:stress-test}

\begin{figure}[t]
    \centering
    \includegraphics[width=0.99\linewidth]{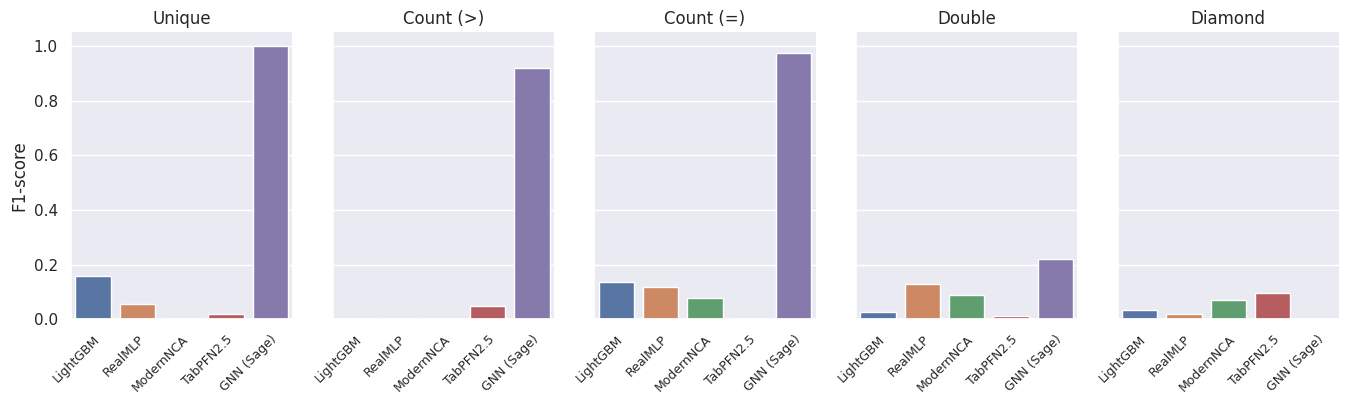}
    \caption{F1 score on the stress-test set of the synthetic transactions data. The stress-test set removes the row-local cues that correlate with the labels in the original test split (Appendix~\ref{app:stress-set}).}
    \label{fig:stress-set}
\end{figure}

Table~\ref{tab:logical-tasks} also contains scores that the theory rules out: row-local models reach AUC above $0.8$ on \textsc{Count} and \textsc{Diamond}, which they provably cannot represent. Figure~\ref{fig:stress-set} reports F1 on the stress-test set, where the row-local cues behind such scores are removed.

Row-local models drop sharply on all five tasks, so their scores in Table~\ref{tab:logical-tasks} came from row-local shortcuts. The GNN keeps high F1 on \textsc{Unique}, \textsc{Count}, and \textsc{Count}$(=)$, the tasks definable in $\mathrm{GML}^{(\le 2)}$ on $\gamma_{\mathrm{inc}}$, and drops on \textsc{Diamond}, which no bounded-depth MPNN on $\gamma_{\mathrm{inc}}$ can express (Proposition~\ref{prop:inc-fo-strict}). This matches Section~\ref{sec:separations} task by task, with one exception: \textsc{Double} is definable in $\mathrm{GML}^{(\le 3)}$, yet the GNN also drops on it. In the original split, the anchor value $c_2 = a$ is a row-local feature that co-occurs with the positive label, and the trained GNN fits this single-condition proxy instead of the 3-hop pattern. Expressibility bounds what a constructor--predictor pair can represent; whether training reaches that bound is a separate question. Access to inter-row structure does not imply that a model uses it, and the stress test is what makes the difference visible.

\subsection{Case study: a practitioner's walkthrough on \texttt{relbench-trial}} \label{sec:relbench-trial}

We now turn to \textbf{(Q3)} and use the \texttt{relbench-trial} dataset from the RelBench benchmark to illustrate how the framework guides model choice on a concrete relational task. We choose this dataset precisely because it behaves atypically: unlike other RelBench tasks (Appendix~\ref{app:relbench_additional}, where the GNN already outperforms the tabular baseline), \texttt{relbench-trial} consistently favours strong tabular models operating on a single table over graph-based models trained on the full relational entity graph.\footnote{The relational entity graph consists of trivial grables, one for each table, with row nodes connected based on key/foreign-key constraints. It is itself a grable.} This makes it an instructive worked example: the framework should explain \emph{why} tabular models do well here, and \emph{how} a practitioner can systematically extend them.

\textbf{Step 1: Diagnose the regime.}
Table~\ref{tab:clinical_results} reports validation and test scores. A strong
tabular baseline (LightGBM) outperforms graph-based models (RDL, RGT) and is
competitive with KumoRFM, indicating that the target is largely extension-invariant on the
studies table. Yet purely tabular predictors are blind to any residual signal that depends
on inter-row relationships, and the framework predicts that exposing such structure should
help \emph{in addition to}, not instead of, the row-local representation.

\begin{figure}[t]
\centering
\begin{minipage}{0.44\linewidth}
\centering
\input{neurips/neurips-tables-figures/neurips-trial-tab}
\end{minipage}
\hfill
\begin{minipage}{0.48\linewidth}
\centering
\includegraphics[width=0.99\linewidth]{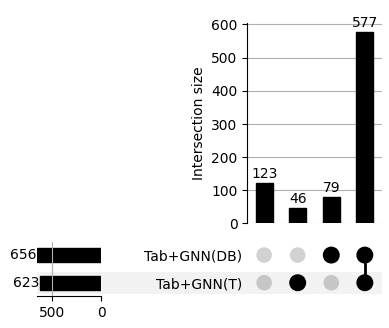}
\captionof{figure}{Overlap and disagreement of predictions across hybrid models on the test set.}
\label{fig:upset_predictions}
\end{minipage}
\vspace{-1ex}
\end{figure}

\textbf{Step 2: Add lightweight inter-row structure.}
Following the hierarchy in Section~\ref{sec:nfa-mpnn}, the cheapest next step is NFA-style aggregation: compile one-hop incidence-graph statistics into explicit features for the same tabular learner. A time-aware variant (Appendix~\ref{app:tab-ta-nfa}) improves over the tabular baseline (Tab+NFA(T) in Table~\ref{tab:clinical_results}) without graph infrastructure at inference.

\textbf{Step 3: Add learned message-passing representations.}
For tasks where shallow aggregation is insufficient, the framework suggests moving to learned representations. We obtain GNN embeddings from message passing on (i)~the incidence graph (Tab+GNN(T)), (ii)~the relational entity graph (Tab+GNN(DB)), and (iii)~both jointly (Tab+GNN(T+DB)), and concatenate them with the original tabular features.\footnote{The original RelBench study also feeds GNN embeddings into LightGBM, but omits the tabular features and does not improve over purely tabular LightGBM. We retain the tabular features and treat graph-based representations as complementary.} All hybrid variants outperform the purely tabular model, the purely graph-based models, and the relational foundation model.

\textbf{Are the structural features actually used?}
Permutation-based feature importance shows that GNN-derived embeddings dominate the top-ranked features across all hybrid models: 7 of the top-15 for Tab+GNN(DB) and 9 for Tab+GNN(T). The tabular learner relies on message-passing representations even in the presence of strong row-local features (Appendix~\ref{app:feature-importance}).

\textbf{Do different structural sources contribute differently?}
Although Tab+GNN(T) and Tab+GNN(DB) achieve similar ROC--AUC, their predictions disagree on 15\% of test instances (Figure~\ref{fig:upset_predictions}). The two constructors expose different aspects of relational structure, i.e., intra-table value sharing versus cross-table key relationships, each informative for a distinct subset of samples.

The framework motivates a procedure: first dentify where the predictive signal lives, then move along the structural ladder until the signal is captured. In our case study, this yields a hybrid model that captures signal missed by either purely tabular or purely relational predictors; the diagnosis correctly identifies that tabular and graph-based representations are complementary here. When the relational signal already dominates (as in most cases of the Relbench benchmark, see Appendix~\ref{app:relbench_additional}) a hybrid model can still help: the inter-row signal is captured by the relational model, and a strong row-local tabular learner helps combining it with the row-local signal that is relevant for the learning task. 

%% file: neurips/neurips-tables-figures/table-synthetic-transactional.tex
\begin{table}[t]
\centering
\caption{Diagnostic comparison on logical tasks in synthetic and transaction datasets. F1 of $0$ corresponds to predicting $0$ for all instances. Results averaged over 5 runs, variance in Appendix \ref{app:full-table-logical-tasks}}
\vspace{1em}
\label{tab:logical-tasks}
\small
\setlength{\tabcolsep}{4pt}
\renewcommand{\arraystretch}{1.05}
\begin{tabular}{l cc cc cc cc cc}
\toprule
\textbf{Model} &
\multicolumn{2}{c}{\textsc{Unique}} &
\multicolumn{2}{c}{\textsc{Count}} &
\multicolumn{2}{c}{\textsc{Count}($=$)} &
\multicolumn{2}{c}{\textsc{Double}} &
\multicolumn{2}{c}{\textsc{Diamond}} \\
\cmidrule(lr){2-3}\cmidrule(lr){4-5}\cmidrule(lr){6-7}\cmidrule(lr){8-9}\cmidrule(lr){10-11}
& AUC & F1 & AUC & F1 & AUC & F1 & AUC & F1 & AUC & F1 \\
\midrule
\multicolumn{11}{c}{\emph{Synthetic transactions}} \\
LightGBM    & 0.566 & 0.123 & 0.889 & 0.000 & 0.675 & 0.314 & 0.789 & 0.729 & 0.846 & 0.682 \\
RealMLP     & 0.516 & 0.123 & 0.933 & 0.733 & 0.521 & 0.000 & 0.785 & 0.686 & 0.796 & 0.603 \\
ModernNCA     & 0.597 & 0.000 & 0.862 & 0.569 & 0.544 & 0.000 & 0.720 & 0.670 & 0.798 & 0.593 \\
TabPFN2.5   & 0.823 & 0.025 & 0.962 & 0.775 & 0.340 & 0.055 & 0.670 & 0.699 & 0.868 & \textbf{0.702} \\
GNN (Sage)  & \textbf{0.997} & \textbf{0.986} & \textbf{1.000} & \textbf{1.000} & \textbf{0.988} & \textbf{0.859} & \textbf{0.887} & \textbf{0.777} & \textbf{0.899} & 0.676 \\
\midrule
\multicolumn{11}{c}{\emph{Real-world retail transactions}} \\
LightGBM    & 0.570 & 0.190 & 0.610 & 0.000 & 0.520 & 0.027 & 0.847 & 0.815 & \textbf{0.821} & 0.199 \\
RealMLP     & 0.853 & 0.190 & 0.929 & 0.881 & 0.477 & 0.027 & 0.703 & 0.000 & 0.792 & 0.139 \\
ModernNCA     & 0.749 & 0.060 & 0.968 & 0.522 & 0.438 & 0.000 & 0.910 & 0.857 & 0.725 & 0.104 \\
TabPFN2.5   & 0.793 & 0.404 & 0.827 & 0.001 & 0.758 & 0.023 & 0.463 & 0.003 & 0.659 & 0.007 \\
GNN (Sage)  & \textbf{0.998} & \textbf{0.988} & \textbf{0.999} & \textbf{0.967} & \textbf{0.903} & \textbf{0.171} & \textbf{0.910} & \textbf{0.863} & 0.807 & \textbf{0.251} \\
\bottomrule
\end{tabular}
\end{table}

%% file: neurips/neurips-tables-figures/neurips-trial-tab.tex
\centering
\captionof{table}{Performance comparison on \texttt{relbench-trial}. Results from RDL, RGT and KumoRFM are taken from their respective papers. KumoRFM does not report validation results. Results averaged over 5 runs.}
\label{tab:clinical_results}
\footnotesize
\setlength{\tabcolsep}{4pt}
\begin{tabular}{lcc}
\toprule
\textbf{Model} & \textbf{Val. AUC} & \textbf{Test AUC} \\
\midrule
LightGBM & 0.6830 & 0.7009 \\
RDL (GNN) & 0.6818 & 0.6860 \\
RGT (Graph Transformer) & 0.6678 & 0.6861 \\
KumoRFM & -- & 0.7079 \\
KumoRFM (finetuned) & -- & 0.7116 \\
\midrule
GNN(T) & 0.6837 & 0.7151 \\
Tab+NFA(T) & 0.7024 & 0.7254 \\
Tab+GNN(T) & 0.7011 & 0.7241 \\
\midrule
Tab+GNN(DB) & 0.6936 & 0.7180 \\
Tab+GNN(T+DB) & \textbf{0.7214} & \textbf{0.7486} \\
\bottomrule
\end{tabular}

%% file: log/log-sections/conclusions-log.tex
We introduced \emph{grables}, a framework that separates how a table is lifted to a graph (the constructor) from how predictions are computed on it (the predictor). This places gradient-boosted trees, neighbourhood aggregation, transformer-style tabular models, and relational GNNs on a single axis, and attributes expressive power to the constructor, the predictor, or the inference protocol. Row-local predictors cannot represent targets driven by counts, overlaps, or shared witnesses. Message passing on the incidence grable captures counts and overlaps but not shared witnesses such as \textsc{Diamond}; this is a limit of the constructor, which one added edge type removes (Appendix~\ref{appendix:extended-incidence}). The inference protocol bounds what any pair can use: TabPFN and ModernNCA have constructors expressive enough for our tasks, but their deployed protocol does not provide the rows these targets depend on.

These results give a practical hierarchy: row-local models suffice for extension-invariant targets, neighbourhood aggregation captures shallow inter-row signals, and explicit graph construction with message passing is needed for compositional patterns that fixed aggregations cannot express without exponential feature blow-up. Since fixed aggregations correspond to hand-engineered group and count features, the hierarchy also says when such features suffice and when they provably do not.

Our experiments confirm the predicted separations, and show that aggregate performance is a poor proxy for structural capability: row-local models score well on targets they provably cannot represent. A stress test exposes this gap, and also shows that access to cross-row context does not imply that a model uses it. A negative control on an i.i.d.\ benchmark shows no consistent gains where none are predicted. On relational data, combining tabular features with message-passing embeddings improves over both purely tabular and purely graph-based models, on \texttt{relbench-trial} and on five further RelBench tasks, so tabular and graph-based representations are complementary rather than competing. The practical recipe is to add structure incrementally, and only where the target is extension-sensitive.

\textbf{Limitations and outlook.} Expressibility bounds what a constructor--predictor pair can represent, not what training reaches: as the \textsc{Double} stress test shows, the framework gives no guarantee of learnability. It also does not identify which inter-row structure is useful for a new dataset; constructors are currently chosen by diagnosis, and making this choice adaptive or learned is a direct extension. The expressibility definition also applies to predictor classes beyond MPNNs. Theory and experiments are instantiated for binary classification; the constructor side does not depend on the output space, so we expect no new representational questions for multiclass and regression, and confirming this empirically is left to future work.

%% file: appendix/appendix-A.tex
\newcommand{\Row}{\mathsf{Row}}
\newcommand{\Val}{\mathsf{Val}}

\section{Related Work}\label{sec:related}
\input{appendix/related-work-neurips}
\section{Background}
We denote \emph{multisets} by $\{\!\{\cdot\}\!\}$ and, for a set $X$, we write $\mathsf{mset}(X)$ for the collection of all finite multisets over $X$.

\paragraph{MPNNs and their induced node predictors}
Fix an input graph $G=(V,E_1,\ldots,E_m,\rho)$.
A \emph{$k$-layer (heterogeneous) MPNN} \citep{Gil+2017,Schlichtkrull+2018}  maintains embeddings $\mathbf{h}_v^{(\ell)}\in\mathbb{R}^d$ for each node $v$ and each layer
$\ell\in\{0,1,\dots,k\}$.  Initialization is given by a learnable encoder
$\mathbf{h}_v^{(0)} = \mathsf{INI}(\rho(v),C(v))$.
For each layer $\ell=0,\ldots,k-1$ and edge type $i\in[m]$, it computes
\begin{align*}
\mathbf{m}_{v,i}^{(\ell)}
&= \mathsf{AGG}_i^{(\ell)}\Big(\big\{\!\!\big\{\ \mathsf{MSG}_i^{(\ell)}(\mathbf{h}_v^{(\ell)},\mathbf{h}_u^{(\ell)})
\ \big|\ u\in N_i(v)\ \big\}\!\!\big\}\Big),\\
\mathbf{h}_v^{(\ell+1)}
&= \mathsf{UPD}^{(\ell)}\big(\mathbf{h}_v^{(\ell)},\mathbf{m}_{v,1}^{(\ell)},\ldots,\mathbf{m}_{v,m}^{(\ell)}\big),
\end{align*}
where each $\mathsf{MSG}_i^{(\ell)}$ and $\mathsf{UPD}^{(\ell)}$ is a differentiable learnable map and each
$\mathsf{AGG}_i^{(\ell)}:\mathsf{mset}(\mathbb{R}^{d'})\to\mathbb{R}^{d'}$ is permutation-invariant.

A $k$-layer MPNN induces a (total)
\emph{node classifier} $\Predn^{\mathsf{MPNN}}$ by thresholding
the final readout embedding:
$$
\Predn_G^{\mathsf{MPNN}}:V\to\{0,1\}:v\mapsto \mathbb{I}\!\left[\ \mathsf{READ}\!\left(\mathbf{h}_v^{(k)}\right) > 0\ \right],
$$
where $\mathbb{I}[\cdot]\in\{0,1\}$ is the indicator function and $\mathsf{READ}:\mathbb{R}^{d}\to\mathbb{R}$ is a differentiable and learnable \emph{readout} function.

\label{appsec:background}
\paragraph{First-order logic}
We here define standard first-order logic over the \emph{graphs}
$G=(V,E_1,\ldots,E_m,\rho)$ from the preliminaries.

\smallskip
\noindent
$\blacktriangleright$ \emph{Signature.}
We recall that for a given graph $G=(V,E_1,\ldots,E_m,\rho)$, we use the binary relation symbols
$E_1,\ldots,E_m$ to form the \emph{edge signature} $\sigma_G = \{E_1,\ldots,E_m\}$. To access node features in logic, we fix a \emph{finite} set $\mathcal U_\rho$ of unary predicate symbols.
Each $U\in\mathcal U_\rho$ is intended to represent a Boolean test on the node annotation $(C(v),\rho(v))$
(e.g., a node-type indicator such as $C(v)=C_0$, or a threshold/equality test on a designated attribute when present).
We work over the expanded signature
$\sigma_G^\rho = \sigma_G \cup \mathcal U_\rho$.

For each  graph $G$, every $U\in\mathcal U_\rho$ is interpreted as a set $U^G\subseteq V$.
Concretely, one may view $\mathcal U_\rho$ as fixing a finite family of Boolean functions
$U:\{(C,\eta)\mid C\subseteq\mathcal C\text{ finite},\,\eta\in\dom^C\}\to\{0,1\}$ and set
\[
U^G = \{v\in V \mid U(C(v),\rho(v))=1\}.
\]

\smallskip
\noindent
$\blacktriangleright$ \emph{FO syntax.}
Let $\mathrm{Var}$ be a countable set of variables.
First-order formulas over $\sigma_G^\rho$ are generated by
\[
\varphi := (x=y)\ \mid\ U(x)\ \mid\ E_i(x,y)\ \mid\ \neg\varphi\ \mid\ (\varphi\wedge\psi)\ \mid\ \exists x\,\varphi,
\]
where $x,y\in\mathrm{Var}$, $U\in\mathcal U_\rho$, and $i\in[m]$.
We use the usual abbreviations for $\vee,\to,\forall$.

\smallskip
\noindent
$\blacktriangleright$ \emph{Free variables.}
The set $\mathrm{fv}(\varphi)\subseteq \mathrm{Var}$ of \emph{free variables} of an FO formula $\varphi$ over $\sigma_G^\rho$
is defined inductively by
\begin{align*}
\mathrm{fv}(x=y) &= \{x,y\},\\
\mathrm{fv}(U(x)) &= \{x\},\\
\mathrm{fv}(E_i(x,y)) &= \{x,y\},\\
\mathrm{fv}(\neg\varphi) &= \mathrm{fv}(\varphi),\\
\mathrm{fv}(\varphi\wedge\psi) &= \mathrm{fv}(\varphi)\cup \mathrm{fv}(\psi),\\
\mathrm{fv}(\exists x\,\varphi) &= \mathrm{fv}(\varphi)\setminus\{x\}.
\end{align*}
A formula $\varphi$ is \emph{(FO-)unary} if $\mathrm{fv}(\varphi)\subseteq\{x\}$ for some variable $x$; in this case we often write $\varphi(x)$.

\smallskip
\noindent
$\blacktriangleright$ \emph{FO semantics.}
A \emph{variable assignment} is a map $\alpha:\mathrm{Var}\to V$.
Satisfaction $G,\alpha\models\varphi$ is defined inductively:
\begin{align*}
G,\alpha\models (x=y) &\iff \alpha(x)=\alpha(y),\\
G,\alpha\models U(x) &\iff \alpha(x)\in U^G,\\
G,\alpha\models E_i(x,y) &\iff (\alpha(x),\alpha(y))\in E_i,\\
G,\alpha\models \neg\varphi &\iff \text{not }(G,\alpha\models\varphi),\\
G,\alpha\models (\varphi\wedge\psi) &\iff (G,\alpha\models\varphi)\ \text{and}\ (G,\alpha\models\psi),\\
G,\alpha\models \exists x\,\varphi &\iff \exists v\in V:\ G,\alpha[x\mapsto v]\models \varphi,
\end{align*}
where $\alpha[x\mapsto v]$ is $\alpha$ updated to map $x$ to $v$.
If $\varphi(x)$ has exactly one free variable $x$, we write $G\models \varphi(v)$ to mean
$G,\alpha\models\varphi$ for any (equivalently, every) assignment $\alpha$ with $\alpha(x)=v$.

\paragraph{Graded modal logic}
We also introduce the graded modal logic (GML) fragment of FO.

\smallskip
\noindent
$\blacktriangleright$ \emph{Counting quantifiers.}
For $N\in\mathbb N$, we write $\exists^{\ge N}y\,\theta(y)$ as an abbreviation for the FO formula expressing
that at least $N$ \emph{distinct} witnesses satisfy $\theta$:
\[
\exists y_1\cdots \exists y_N\Bigl(\bigwedge_{p\neq q} y_p\neq y_q\ \wedge\ \bigwedge_{p=1}^N \theta(y_p)\Bigr).
\]
Equivalently, for any assignment $\alpha$,
$
G,\alpha\models \exists^{\ge N}y\,\theta(y)
$ if and only if 
\[\bigl|\{v\in V \mid G,\alpha[y\mapsto v]\models \theta(y)\}\bigr|\ \ge\ N.
\]

\smallskip
\noindent
$\blacktriangleright$ \emph{GML syntax and semantics.}
The \emph{graded modal logic} (GML) fragment over $\sigma_G^\rho$ consists of unary formulas $\varphi(x)$ generated by 
$\varphi(x)\!:=$
\[
U(x)\ \mid\ \neg\varphi(x)\ \mid\ (\varphi(x)\wedge\psi(x))\ \mid\
\exists^{\ge N}y\,\bigl(E_i(x,y)\wedge \varphi(y)\bigr),
\]
where $U\in\mathcal U_\rho$, $i\in[m]$, and $N\in\mathbb N$. Clearly, GML formulas
inherit the semantics of FO formulas and the counting quantifiers given earlier.

\smallskip
\noindent
$\blacktriangleright$ \emph{Modal depth.}
The modal depth $\mathrm{md}(\varphi)$ of a GML formula is defined by
$\mathrm{md}(U(x)) = 0$,
$\mathrm{md}(\neg\varphi) = \mathrm{md}(\varphi)$,$\mathrm{md}(\varphi\wedge\psi) = \max\{\mathrm{md}(\varphi),\mathrm{md}(\psi)\}$, and
$$\mathrm{md}\!\left(\exists^{\ge N}y\,(E_i(x,y)\wedge \varphi(y))\right) = 1+\mathrm{md}(\varphi).$$
We write $\mathrm{GML}^{(\le k)}$ for the set of GML formulas of modal depth at most $k$.

\section{Model-induced grables for CARTE, TARTE, ModernNCA, TabPFN, and NFA}
\label{appsec:grables}
Here we instantiate graph constructors for four prominent tabular learning methods by making their \emph{tokenization} and \emph{connectivity} explicit as a grable. We also formulate Neighbourhood Feature Aggregation as a grable.

Throughout, let $C=\{c_1,\dots,c_m\}$ be a schema and $T\subseteq\dom^C$ a $C$-table.
We use $\rho$ for node features, and we allow edge features/labels (e.g.\ column-name embeddings) as part of the graph representation.

\subsection{CARTE}
\label{app:carte_grable}

CARTE \citep{kim2024carte} represents each row as a small multi-relational \emph{graphlet}: a star-shaped graph with a central row token and one leaf per non-missing cell.
We formalize this as a constructor that produces a single grable for the whole table by taking the disjoint union of row graphlets.

\begin{example}[CARTE grable]
Let $C=\{c_1,\dots,c_m\}$ and $T$ be a $C$-table.
The \emph{CARTE grable} of $(C,T)$ is the  graph
\[
G^{\mathrm{CARTE}}_{C,T}=\bigl(V,\;E_1,\dots,E_m,\;\rho,\;\eta\bigr)
\]
defined as follows. We allow also for an edge feature map $\eta$ here.
For each row $r\in T$, let $I(r)=\{\,i\in[m]\mid r[c_i]\neq\mathsf{NaN}\,\}$ and create a \emph{row center} node $v_r$ and one \emph{cell} node $c_{r,i}$ for each $i\in I(r)$.
Thus
\[
V=\{v_r:r\in T\}\ \cup\ \{c_{r,i}: r\in T,\ i\in I(r)\}.
\]
For each $i\in[m]$, the $i$-th edge relation is
\[
E_i=\{(v_r,c_{r,i}) : r\in T,\ i\in I(r)\}
\quad (\text{optionally also } (c_{r,i},v_r)).
\]
Node features encode values (optionally conditioned on the column):
\[
\rho(c_{r,i})=\phi(c_i,r[c_i]),\qquad \rho(v_r)=r,
\]
for any fixed featurizer $\phi$.
Edge features carry column semantics via
\[
\eta\bigl((v_r,c_{r,i})\bigr)=\mathrm{LM}(c_i),
\]
for some language model embedding LM.
We write $\gamma_{\mathrm{CARTE}}$ for the corresponding constructor.\hfill$\blacktriangleleft$
\end{example}

\subsection{TARTE}
\label{app:tarte_grable}

TARTE \citep{kim2025table} is a transformer-style tabular backbone that forms representations from \emph{column--cell pairs}: each entry is interpreted in the context of its column name, and the model aggregates information across a row via attention to produce row-level predictions.

\begin{example}[TARTE grable]
Let $C=\{c_1,\dots,c_m\}$ and $T$ be a $C$-table.
The \emph{TARTE grable} of $(C,T)$ is the  graph
\[
G^{\mathrm{TARTE}}_{C,T}=\bigl(V,\;E_{\mathrm{attn}},\;E_{\mathrm{row}},\;\rho\bigr)
\]
defined as follows.
The node set consists of row nodes and column--cell token nodes:
\[
V=\{v_r:r\in T\}\ \cup\ \{t_{r,i}: r\in T,\ i\in[m]\}.
\]
Token features encode the column name together with the cell value,
\[
\rho(t_{r,i})=\phi(c_i,r[c_i]),
\qquad
\rho(v_r)=r,
\]
for some fixed featurizer $\phi$.
Edges encode (permissible) attention within each row:
\[
E_{\mathrm{attn}}=\{(t_{r,i},t_{r,j}) : r\in T,\ i,j\in[m]\}.
\]
We also connect each row node to its tokens via
\[
E_{\mathrm{row}}=\{(v_r,t_{r,i})\mid r\in T,\ i\in[m]\}
\quad(\text{optionally also }(t_{r,i},v_r)).
\]
We write $\gamma_{\mathrm{TARTE}}$ for the corresponding constructor.\hfill$\blacktriangleleft$
\end{example}

\subsection{ModernNCA}
\label{app:modernnca_grable}
ModernNCA \citep{modern-nca} is naturally viewed as a retrieval-based tabular method whose predictions depend on comparing a row against other rows in the visible context. In our framework, the essential structural bias is therefore direct row-to-row access. Abstracting away the particular similarity or scoring mechanism, we model this by a constructor whose nodes are the rows themselves and whose edges make all pairwise row interactions available. 
\begin{example}[ModernNCA grable]
Let $C=\{c_1,\dots,c_m\}$ and $T$ be a $C$-table.
The \emph{ModernNCA grable} of $(C,T)$ is the graph
\[
G^{\mathrm{MNCA}}_{C,T} \;=\; (V,\;E,\;\rho),
\]
with row nodes $V = \{v_r \mid r \in T\}$ and edges connecting every pair of distinct rows,
\[
E = \{(v_r, v_s) \mid r,s \in T,\ r \neq s\}.
\]
Node features store the full row,
\[
\rho(v_r) = r \in \mathrm{val}^{C}.
\]
We write $\gamma_{\mathrm{MNCA}}$ for the corresponding constructor.\hfill$\blacktriangleleft$
\end{example}

\subsection{TabPFN}
\label{app:tabpfn_grable}

TabPFN \citep{tabpfn_nature} consumes an entire table as a single set of tokens.
A useful graph-level view is that each transformer block restricts self-attention to either
(i) tokens within the same \emph{row} (feature interactions) or
(ii) tokens within the same \emph{column} (cross-row information exchange for a fixed feature).
Thus, TabPFN can be seen as message passing on a  graph whose edges encode permissible attention neighborhoods, together with an alternation schedule over edge types across layers.

\begin{example}[TabPFN grable]
Let $C=\{c_1,\dots,c_m\}$ and $T$ be a $C$-table.
The \emph{TabPFN grable} of $(C,T)$ is the graph
\[
G^{\mathrm{TabPFN}}_{C,T} \;=\; (V,\;E^{\mathrm{row}},\;E^{\mathrm{col}},\;\rho),
\]
with $V = V_{\mathrm{cell}}\cup V_{\mathrm{row}}$, where
\[
V_{\mathrm{cell}} = \{c_{r,i}\mid r\in T,\ i\in[m]\},
\qquad
V_{\mathrm{row}} = \{v_r\mid r\in T\}.
\]
Row- and column-wise attention relations are
\begin{align*}
E^{\mathrm{row}} &= \{(c_{r,i},c_{r,j})\mid r\in T,\ i,j\in[m]\}
\ \cup\ {}\\
& \hspace{1cm}\{(v_r,c_{r,i}),(c_{r,i},v_r)\mid r\in T,\ i\in[m]\},\\
E^{\mathrm{col}} &= \{(c_{r,i},c_{s,i})\mid r,s\in T,\ i\in[m]\}.
\end{align*}
Node features are initialized by
\[
\rho(c_{r,i}) = \psi(c_i,\,r[c_i]),\qquad
\rho(v_r) = r,
\]
for some featurizer $\psi$. 
We write $\gamma_{\mathrm{TabPFN}}$ for the corresponding constructor.\hfill$\blacktriangleleft$
\end{example}

\subsection{NFA}
\label{app:nfa_grable}

Neighborhood Feature Aggregation (NFA) \citep{bazhenov2025GraphLand} is a \emph{derived} constructor: it requires a starting grable $G^\gamma_{C,T}$ to define $1$-hop neighborhoods, but its output is a \emph{trivial grable} so that downstream learners can be purely tabular.

\begin{example}[NFA grable (trivial output)]
Fix a base constructor $\gamma$ and let
\[
G^\gamma_{C,T}=(V,\;E_1,\dots,E_m,\;\rho)
\]
be its grable on $(C,T)$ with row nodes
$V_{\mathrm{row}}=\{v_r\mid r\in T\}$.
Write $E=\bigcup_{i=1}^m E_i$.

NFA is applied to a row-neighbourhood relation on $V_{\mathrm{row}}$.
When the base grable already contains row-to-row edges, we take
\[
N(v)=\{u\in V_{\mathrm{row}}\mid (v,u)\in E\}.
\]
For bipartite grables such as the incidence grable, which contain no direct
row-to-row edges, we instead use the induced row projection: two row nodes
$v_r$ and $v_{r'}$ are neighbours whenever they are connected by a two-hop path
through an auxiliary node. Equivalently,
\[
N(v_r)
=
\{v_{r'}\in V_{\mathrm{row}}\mid
\exists z\in V\setminus V_{\mathrm{row}}:
(v_r,z)\in E \text{ and } (z,v_{r'})\in E\}.
\]
For example, for the incidence grable, NFA aggregates over rows sharing at least one
column--value node, which is precisely the one-hop neighbourhood in the
row-projected incidence graph.

Define the NFA augmentation for a row node $v\in V_{\mathrm{row}}$ as follows.
For any scalar feature coordinate $z$ (numerical, or a one-hot indicator derived from a categorical feature),
let
\[
S_z(v)=\{\, z(u)\mid u\in N(v),\ z(u)\neq \mathsf{NaN}\,\}.
\]
Then $\mathrm{NFA}(v)$ concatenates:
(i) for each numerical coordinate $z$: $\mathrm{mean}(S_z(v))\Vert \min(S_z(v))\Vert \max(S_z(v))$;
(ii) for each one-hot indicator $z$ of each categorical feature: $\mathrm{mean}(S_z(v))$.

Set the enriched row features to
\[
\rho^{\mathrm{NFA}}(v)\ =\ \rho(v)\ \Vert\ \mathrm{NFA}(v)\qquad (v\in V_{\mathrm{row}}).
\]
Finally, the \emph{NFA grable} is the \emph{trivial} grable on row nodes with these features:
\[
G^{\mathrm{NFA}}_{C,T}\ =\ \bigl(V_{\mathrm{row}},\;\emptyset,\;\rho^{\mathrm{NFA}}\bigr).
\]
We write $\gamma_{\mathrm{NFA}}(\gamma)$ for the resulting constructor.\hfill$\blacktriangleleft$
\end{example}

\section{Grabular expressiveness}\label{appsec:results}

\subsection{Proof of Proposition~\ref{lem:mpnn-triv-rowwise}}
Fix $k\in\mathbb N$. We need to show that for any $k$-layer MPNN node classifier on $G^{\mathrm{triv}}_{C,T}$, there exists a function $f:\dom^C\to\{0,1\}$ (determined by the learned parameters) such that for all $r\in T$, \[ \Predn^{\mathrm{MPNN}}_{G^{\mathrm{triv}}_{C,T}}(v_r)=F(r). \]

Indeed, by definition of the trivial constructor, $G^{\mathrm{triv}}_{C,T}$ has no edges of any type, so for every row node $v_r$ and every layer $\ell$ the neighbour multiset used by the aggregator is empty. Hence the aggregated message at layer $\ell$ is the same constant for all nodes:
\[
\mathbf{m}^{(\ell)}(v_r)\;=\;\mathsf{Agg}^{(\ell)}(\emptyset)\;=:\mathbf{c}_\ell,
\]
which does not depend on $T$ or on $r$.
We can now write the embeddings computed by the MPNN as
\[
\mathbf h^{(0)}_{v_r}=g(r),\qquad
\mathbf h^{(\ell+1)}_{v_r}=\mathsf{Upd}^{(\ell)}\!\bigl(\mathbf h^{(\ell)}_{v_r},\, \mathbf{m}^{(\ell)}(v_r)\bigr)
\quad(\ell=0,\dots,k-1),
\]
and final prediction
\[
\Predn^{\mathrm{MPNN}}_{G^{\mathrm{triv}}_{C,T}}(v_r)=\mathsf{Read}\!\bigl(\mathbf h^{(k)}_{v_r}\bigr).
\]
Substituting $\mathbf{m}^{(\ell)}(v_r)=\mathbf{c}_\ell$ yields
\[
\mathbf h^{(\ell+1)}_{v_r}=\mathsf{Upd}^{(\ell)}\!\bigl(\mathbf h^{(\ell)}_{v_r},\, \mathbf{c}_\ell\bigr),
\]
so the evolution of $\mathbf h_{v_r}^{(\ell)}$ depends only on the initial feature $g(r)$.

Define functions $g_\ell:\dom^C\to\mathbb R^d$ recursively by
\[
g_0(r):=g(r),\qquad g_{\ell+1}(r):=\mathsf{Upd}^{(\ell)}\!\bigl(g_\ell(r),\,\mathbf{c}_\ell\bigr).
\]
Then by induction on $\ell$ we have $\mathbf h^{(\ell)}_{v_r}=g_\ell(r)$ for all $\ell\le k$, and in particular
$\mathbf h^{(k)}_{v_r}=g_k(r)$.
Now define the row-wise classifier
\[
f:\dom^C\to\{0,1\},\qquad f(r):=\mathsf{Read}\!\bigl(g_k(r)\bigr).
\]
Therefore, for every $r\in T$,
\[
\Predn^{\mathrm{MPNN}}_{G^{\mathrm{triv}}_{C,T}}(v_r)
=\mathsf{Read}\!\bigl(\mathbf h^{(k)}_{v_r}\bigr)
=\mathsf{Read}\!\bigl(g_k(r)\bigr)
=f(r),
\]
as claimed.

\subsection{Proof of Proposition~\ref{prop:inc-not-rowlocal}}
We need to show that each predictor in Figure~\ref{fig:logical-tasks-incidence} is not row-local. We exhibit, for each task, an instance where the label of a fixed row $r$ changes after adding rows (while keeping $r$ itself unchanged), hence the predictor is not row-local.

\smallskip
\noindent\textsc{Unique.}
Let $T=\{r\}$ where $r[c_i]=v$ for some column $c_i$. Then $v$ appears only in $r$, so \textsc{Unique}$(r)=1$.
Let $T':=\{r,s\}$ where $s[c_i]=v$ and $s$ agrees with $r$ on $c_i$ (other columns arbitrary). Then $v$ is no longer unique to $r$, so \textsc{Unique}$(r)=0$ in $T'$.

\smallskip
\noindent\textsc{Count}$_N$.
Let $T=\{r\}$, so there are $0$ other rows sharing any value with $r$, hence \textsc{Count}$_N(r)=0$.
Add $N$ fresh rows $s_1,\dots,s_N$ with $s_j[c_i]=r[c_i]$ for the designated column $c_i$ (others arbitrary). In the enlarged table $T'$ we have at least $N$ other rows sharing that value with $r$, hence \textsc{Count}$_N(r)=1$.

\smallskip
\noindent\textsc{Double.}
Pick $T=\{r\}$ so no witness row exists and \textsc{Double}$(r)=0$.
Add a row $s$ that shares the designated value with $r$ and satisfies the required local constraint (unary predicate on $s$'s own features). Then \textsc{Double}$(r)=1$ in $T':=\{r,s\}$.

\smallskip
\noindent\textsc{Diamond.}
Pick $T=\{r\}$ so no other row exists and \textsc{Diamond}$(r)=0$.
Add a row $s$ such that $s[c_i]=r[c_i]$ and $s[c_j]=r[c_j]$ for two designated columns $c_i\neq c_j$. Then $s$ witnesses the required double-overlap and \textsc{Diamond}$(r)=1$ in $T':=\{r,s\}$.

\smallskip
In all four cases, $\Pred_{C,T}(r)\neq \Pred_{C,T'}(r)$ for some $T\subseteq T'$, so none of these predictors is row-local. Equivalently, none lies in $\mathsf{Expr}(\{\gamma_{\mathrm{triv}}\},\mathcal P)$ for any $\mathcal P$.

\subsection{Proof of Proposition~\ref{prop:inc-mpnn-modal}}
Let $\mathcal P_{\mathrm{MPNN}}^{(k)}$ denote $k$-layer MPNNs. We need to show that
\begin{compactitem}
\item $\textsc{Unique},\ \textsc{Count}_N\ \in\ \mathsf{Expr}(\{\gamma_{\mathrm{inc}}\},\mathcal P_{\mathrm{MPNN}}^{(2)})$,
\item $\textsc{Double}\ \in\ \mathsf{Expr}(\{\gamma_{\mathrm{inc}}\},\mathcal P_{\mathrm{MPNN}}^{(3)})$.
\end{compactitem}

We write the three targets as $\mathrm{GML}$ formulas using the counting modality
$\exists^{\ge N}y\,(E_i(x,y)\wedge \varphi(y))$ from the grammar.
Let $\Row(\cdot)$ and $\Val(\cdot)$ be unary predicates distinguishing row and value nodes in $\gamma_{\mathrm{inc}}$
(and let $P(\cdot)\in\mathcal U_\rho$ be the unary “local condition” used in \textsc{Double}).

\smallskip
\noindent\textsc{Unique.}
Fix the intended meaning: $r$ has some column--value neighbor that is incident to no other row.
For each column type $i$, define
\[
\psi_{\mathrm{uniq},i}(x)\ :=\
\exists^{\ge 1}y\Bigl(E_i(x,y)\wedge \Val(y)\wedge \neg\exists^{\ge 2}z\bigl(E_i(y,z)\wedge \Row(z)\bigr)\Bigr),
\]
and let $\psi_{\mathrm{uniq}}(x):=\bigvee_i \psi_{\mathrm{uniq},i}(x)$.
At a row node $x=v_r$, the inner negated counting condition says that the chosen value node $y=u_{i,v}$
has fewer than two adjacent row nodes via $E_i$, i.e.\ it is adjacent to exactly one row (namely $v_r$).
Thus $\psi_{\mathrm{uniq}}$ defines \textsc{Unique}. Its modal depth is $2$.

\smallskip
\noindent\textsc{Count}$_N$.
Assume the designated column type is $i$ (otherwise disjoin over the permitted choices).
Define
\[
\psi_{\mathrm{cnt},N}(x)\ :=\
\exists^{\ge 1}y\Bigl(E_i(x,y)\wedge \Val(y)\wedge \exists^{\ge (N+1)}z\bigl(E_i(y,z)\wedge \Row(z)\bigr)\Bigr).
\]
At $x=v_r$, this picks a value node $y=u_{i,v}$ adjacent to $r$ whose $E_i$-row-neighborhood has size at least $N+1$,
i.e.\ at least $N$ \emph{other} rows share that value with $r$. This has modal depth $2$.

\smallskip
\noindent\textsc{Double.}
Assume the task is witnessed by a length-$3$ pattern
$v_r - u_{i,v} - v_s - u_{j,w}$ where $v_s$ satisfies the unary predicate $P$,
and where the witness row $v_s$ must be distinct from the query row $v_r$.

First define the GML subformula
\[
Q(z)
\ :=\
\Row(z)\wedge P(z)\wedge
\exists^{\ge 1}w\bigl(E_j(z,w)\wedge \Val(w)\bigr).
\]
Thus $Q(z)$ says that $z$ is an eligible witness row: it is a row, satisfies
the local condition $P$, and has a $j$-value neighbour.

Since GML has no direct inequality test inside the modal step, we enforce
$s\neq x$ by branching on whether the query row itself satisfies $Q$:
\[
\begin{aligned}
\psi_{\mathrm{dbl}}(x)
\ :=\ 
&\Bigl(
\neg Q(x)\ \wedge\
\exists^{\ge 1}y\Bigl(
E_i(x,y)\wedge \Val(y)\wedge
\exists^{\ge 1}s\bigl(E_i(y,s)\wedge Q(s)\bigr)
\Bigr)
\Bigr)
\\
&\vee
\Bigl(
Q(x)\ \wedge\
\exists^{\ge 1}y\Bigl(
E_i(x,y)\wedge \Val(y)\wedge
\exists^{\ge 2}s\bigl(E_i(y,s)\wedge Q(s)\bigr)
\Bigr)
\Bigr).
\end{aligned}
\]
If $Q(x)$ is false, then the query row $x$ cannot itself be counted as an
eligible witness, so one $Q$-row in the $E_i$-neighbourhood of the selected
value node suffices. If $Q(x)$ is true, then $x$ is itself one of the
$Q$-rows adjacent to that value node; requiring at least two such rows forces
the existence of another row $s\neq x$. Hence the formula defines
\textsc{Double} with a distinct witness row. Since $Q$ has modal depth $1$,
the full formula has modal depth $3$.
\smallskip
Hence \textsc{Unique} and \textsc{Count}$_N$ are definable in $\mathrm{GML}^{(\le 2)}$ and \textsc{Double} in $\mathrm{GML}^{(\le 3)}$
on $\gamma_{\mathrm{inc}}$.
By Theorem~\ref{thm:fo-mpnn-gml}, these properties are realized by $2$-layer (resp.\ $3$-layer) MPNNs, giving the claimed
membership in $\mathsf{Expr}(\{\gamma_{\mathrm{inc}}\},\mathcal P_{\mathrm{MPNN}}^{(k)})$.

\subsection{Proof of Proposition~\ref{prop:inc-fo-strict}}

We prove the claim for the color-only incidence constructor. That is, node
features are one-hot encodings of a finite set of structural colors:
row nodes have color $\Row$, and value nodes for column $c_\ell$ have color
$\Val_\ell$. Actual cell values are not exposed as node colors or row-node
features; equality of values is represented only by incidence to the same
value node. 

We need to show that for every $k\in\mathbb N$,
\[
\textsc{Diamond}\notin
\mathsf{Expr}(\{\gamma_{\mathrm{inc}}\},\mathcal P_{\mathrm{MPNN}}^{(k)}).
\]

First observe that \textsc{Diamond} is FO-definable on the incidence graph.
For two designated columns $c_i\neq c_j$, it is defined by
\begin{multline*}
\delta_{\Diamond}(x)
:=
\exists s\,\exists a\,\exists b\;\bigl(
\Row(x)\wedge \Row(s)\wedge s\neq x \wedge \Val_i(a)\wedge \Val_j(b)
\\
\wedge E_i(x,a)\wedge E_i(s,a)
\wedge E_j(x,b)\wedge E_j(s,b)
\bigr).    
\end{multline*}
This says that there is another row $s$ sharing both the $c_i$-value and
the $c_j$-value of $x$.

We now show that this property is not GML-definable. Consider a schema
containing the two designated columns $c_i,c_j$ and an auxiliary identifier column $c_0$ used only to make rows
distinct. The values in the identifier column will not be used by
\textsc{Diamond}.

Let $\alpha,\beta,\gamma,\delta$ be distinct values for the designated
columns, and let all identifier values be fresh. Define the positive table
\[
T^+
=
\{r,s\}
\]
by
\[
r[c_i]=\alpha,\quad r[c_j]=\beta,
\qquad
s[c_i]=\alpha,\quad s[c_j]=\beta,
\]
with $r[c_0]\neq s[c_0]$. In the incidence
graph, ignoring the identical pendant identifier leaves, the subgraph induced
by columns $c_i,c_j$ is the $4$-cycle
\[
v_r - u_{i,\alpha} - v_s - u_{j,\beta} - v_r.
\]
Hence $v_s$ witnesses \textsc{Diamond} at $v_r$, so $
G^{\mathrm{inc}}_{C,T^+},v_r \models \delta_{\Diamond}$. Define the negative table
\[
T^-
=
\{r,s_1,t,s_2\}
\]
by
\[
\begin{array}{c|cc}
 & c_i & c_j\\
\hline
r   & \alpha & \beta\\
s_1 & \alpha & \gamma\\
t   & \delta & \gamma\\
s_2 & \delta & \beta .
\end{array}
\]
Again, give all rows fresh identifier values for the id column $c_0$ (not shown). As  before, in the incidence graph, ignoring the pendant identifier leaves, the
subgraph induced by columns $c_i,c_j$ is the $8$-cycle
\[
v_r - u_{i,\alpha} - v_{s_1} - u_{j,\gamma}
- v_t - u_{i,\delta} - v_{s_2} - u_{j,\beta} - v_r.
\]
Here $v_r$ shares its $c_i$-value with $v_{s_1}$ and its $c_j$-value with
$v_{s_2}$, but no single row shares both values with $v_r$. Therefore $
G^{\mathrm{inc}}_{C,T^-},v_r \not\models \delta_{\Diamond}$.

However, the two rooted colored incidence graphs are indistinguishable by
GML. Indeed, every row node has exactly one $E_i$-neighbour of color
$\Val_i$ and exactly one $E_j$-neighbour of color $\Val_j$; every
$\Val_i$-node has exactly two incident row neighbours; and every
$\Val_j$-node has exactly two incident row neighbours. If the auxiliary
identifier column $c_0$ is present, every row additionally has exactly one
pendant $E_0$-neighbour of color $\Val_0$, and every such $\Val_0$-node has
exactly one row neighbour, in both graphs.
Thus the relation pairing nodes with the same structural color is a graded
bisimulation between the two rooted graphs. Equivalently, their rooted
colored unravellings are isomorphic at every modal depth. Hence they satisfy
the same GML formulas:
\[
G^{\mathrm{inc}}_{C,T^+},v_r
\equiv_{\mathrm{GML}}
G^{\mathrm{inc}}_{C,T^-},v_r.
\]
Since $\delta_{\Diamond}$ distinguishes the two roots, \textsc{Diamond} is
not definable in GML.

By Theorem~\ref{thm:fo-mpnn-gml}, the FO node classifiers expressible by
bounded-depth MPNNs are precisely the GML-definable ones. Therefore no
$k$-layer MPNN realizes \textsc{Diamond} on the color-only incidence
constructor, for any $k\in\mathbb N$.

\subsection{Extended incidence graph}
\label{appendix:extended-incidence}
We show that, by using an extended incidence constructor we can express 
\textsc{Diamond} using MPNNs. Indeed, fix two designated columns $c_i\neq c_j$.
Define an extended constructor $\gamma_{\mathrm{inc}}^{(2)}$ as follows.
Starting from the incidence grable $G^{\gamma_{\mathrm{inc}}}_{C,T}$, add a new node $p_{\alpha,\beta}$
for every pair $(\alpha,\beta)$ that occurs in columns $(c_i,c_j)$ of some row in $T$, and add a new edge type $E_{ij}$
such that
\[
(v_r,p_{\alpha,\beta})\in E_{ij}
\quad\Longleftrightarrow\quad
r[c_i]=\alpha\ \wedge\ r[c_j]=\beta.
\]
\newcommand{\Pair}{\mathsf{Pair}}
Let $\Pair(\cdot)\in\mathcal U_\rho$ be a unary predicate true exactly on these pair nodes, and this predicate is now available for the logic GML.

On $G^{\gamma_{\mathrm{inc}}^{(2)}}_{C,T}$, the task \textsc{Diamond} is defined by the $\mathrm{GML}$ formula
\[
\psi_{\Diamond}(x)\ :=\
\exists^{\ge 1}p\Bigl(E_{ij}(x,p)\wedge \Pair(p)\wedge
\exists^{\ge 2}y\bigl(E_{ij}(p,y)\wedge \Row(y)\bigr)\Bigr).
\]
Indeed, for a row node $x=v_r$, the outer quantifier selects the pair node
$p_{r[c_i],\,r[c_j]}$ representing the joint value $(r[c_i],r[c_j])$.
The inner counting condition requires that this pair node be incident (via $E_{ij}$) to at least two row nodes,
namely $v_r$ and some other $v_s$, which is equivalent to the existence of a single other row $s\neq r$
sharing \emph{both} values with $r$.

Thus $\psi_{\Diamond}\in \mathrm{GML}^{(\le 2)}$, and by Theorem~\ref{thm:fo-mpnn-gml} it follows that
\[
\textsc{Diamond}\in \mathsf{Expr}(\{\gamma_{\mathrm{inc}}^{(2)}\},\mathcal P_{\mathrm{MPNN}}^{(2)}).
\]
This is again a nice  illustration of the constructor/grable formalism: a minimal change to the constructor (turning a two-hop conjunction into a single edge type) immediately collapses an FO-only pattern into shallow GML,
and hence into constant-depth message passing.

\input{appendix/appendix_tabular-nfa}

\input{appendix/appendix-datasets}

%% file: appendix/related-work-neurips.tex
\paragraph{Tabular learning.}
Learning on tabular data has been dominated by a set of highly effective model classes,
including gradient-boosted decision trees (GBDTs) such as LightGBM, XGBoost, and
CatBoost~\cite{lightgbm_neurips,xgboost,catboost}, as well as carefully tuned multilayer
perceptrons designed for tabular inputs, including RealMLP and
TabM~\cite{realmlp_neurips,gorishniy2025tabm}. These models form the basis of many
practical systems and are commonly paired with AutoML tools for hyperparameter tuning,
model selection, and ensembling~\cite{h2o,oboe,erickson2020autogluon}.

A shared characteristic of these approaches is that predictions are defined at the level of
individual rows, using only the attributes of the row itself. This modeling choice aligns
naturally with i.i.d.\ benchmark formulations and has proven effective across a wide range
of applications. This alignment is reinforced by common tabular benchmark protocols, which
typically rely on random splits and implicitly assume i.i.d.\ data, even when this assumption
may not strictly hold in practice~\citep{erickson2025tabarena}.

Our work departs from this setting by constructing prediction tasks whose targets depend on
relationships between rows. Rather than modifying the underlying tabular architectures, we
use these tasks to systematically probe how standard tabular models behave when the
assumptions under which they are typically evaluated no longer hold, and to motivate the
need for representations that can access inter-row information.

\paragraph{Transformer-based models for tabular data.}
A parallel line of work adapts transformer architectures to tabular inputs, motivated by the
success of attention-based models in natural language processing. Models such as
FT-Transformer~\cite{ft-transformer} and TabTransformer~\cite{tabtransformer} primarily
apply attention over features within a single row, and thus remain row-local at inference
time. More broadly, transformer-based tabular models typically rely on serializing tables
into token sequences and injecting structural bias through positional embeddings, masking
schemes, or token-type indicators, rather than through explicit relational structure.

Recent surveys~\citep{transformers4tabdata} provide a comprehensive taxonomy of these
design choices, covering table serialization strategies, positional encodings, and
architectural adaptations. While such models may implicitly capture statistical regularities
across rows through attention, the induced notion of structure depends on tokenization and
architectural details and is not explicitly represented. As a result, it is difficult to
characterize which inter-row dependencies these models can or cannot represent in
principle, or to disentangle genuine relational reasoning from effects of model capacity or
training data scale. Our work is complementary to this literature in that we introduce an
explicit abstraction separating table-to-structure construction from prediction, enabling a
principled analysis of when and why access to inter-row structure matters.

\paragraph{Tabular foundation models.}
Recent tabular foundation models such as TabPFN~\citep{tabpfn_nature},
TabICL~\citep{qu2025tabicl}, and TabDPT~\citep{ma2025tabdpt} leverage attention
mechanisms to enable interactions between rows at inference time, thereby relaxing the
row-locality assumption that underlies most traditional tabular predictors. The primary
motivation of this line of work is to provide zero- or few-shot, ready-to-use models that
generalize across a wide range of tabular datasets, rather than to explicitly study or
characterize the structural limitations induced by row-local
representations.%
\footnote{Not all tabular foundation models aim to introduce cross-row structure; for
instance, CARTE~\citep{kim2024carte} represents each row independently and focuses on
transfer across heterogeneous tables, and TARTE~\citep{kim2025table} similarly produces
row-level representations through semantic pre-training on large knowledge bases. This
further illustrates that lifting row-locality is not the defining objective of this line of work.}
Importantly, cross-row attention does not automatically imply set-based inference.
TabPFN's inter-sample attention explicitly prevents test rows from attending to each other,
allowing only test-to-training attention~\citep{tabpfn_nature}; each test row is therefore
scored against a fixed training context, not against co-occurring test rows. This places
TabPFN in the same inference-protocol class as retrieval-based methods
(see below), despite its richer attention mechanism. Our framework makes this distinction
precise and places these methods in the expressiveness hierarchy: access to cross-row
context is not sufficient for extension-sensitive reasoning if the inference protocol does not
expose the evaluation split as a co-inference context. Accordingly, we include TabPFN in
our experiments and evaluate its behavior on tasks that require reasoning beyond row
locality.

\paragraph{Retrieval-based and inter-sample tabular models.}
A line of work predating tabular foundation models relaxes row-locality through explicit
retrieval or inter-sample attention. SAINT~\citep{saint} and
NPT~\citep{kossen-npt} apply attention across both rows and features at inference time.
NPT is explicitly transductive: it receives the full evaluation set as a single input and permits attention between all datapoints, including test-to-test interactions, placing it in the set-based inference protocol. SAINT similarly computes inter-sample attention over its input batch; while test-to-test interactions are limited to rows within the same batch, this still constitutes set-based inference under a batched evaluation protocol. More recent retrieval-augmented models, TabR~\citep{tabr} and
ModernNCA~\citep{modern-nca}, learn an embedding space in which each test row's prediction is a weighted aggregation over training neighbors, using respectively a sparse
top-$m$ attention and a soft nearest-neighbor rule. These methods are not row-local in our sense: their predictions depend on the training set and change if training data is added or removed. However, they share a different inference protocol from SAINT and NPT: each test row is scored independently against a fixed training reference set, with no edges
between test rows. In our terminology, this is \emph{single-point with reference}, not set-based inference. As a result, TabR and ModernNCA are incapable of representing extension-sensitive targets defined within a sample set, such as
\textsc{Unique} or \textsc{Count}, where the label of a row depends on other rows within
a defined set. This is not a limitation of their architecture but of their inference protocol: the training set acts as a static labeled memory, not a co-inference context. Notably, this ceiling is shared by TabPFN despite its denser attention mechanism, since its deployed inference protocol enforces the same test-row independence.

\paragraph{Graph learning over tabular data.}
Several works apply graph learning to tabular data by first transforming tables into graphs
and then using message-passing models on the resulting structure. A recent
survey~\cite{gnn4td} systematizes these approaches and categorizes table-to-graph
constructions, including graphs where rows are nodes connected by similarity or temporal
relations, feature- or feature-value-centric constructions, and heterogeneous or bipartite
graphs linking rows and attributes. Related work in relational deep learning extends this
paradigm to multi-table databases, where rows correspond to entities in a relational entity
graph connected via primary-foreign key relationships and other schema-defined
links~\cite{RDL_position}. Recent work on task-aware graph construction emphasizes that
this lifting step is not neutral: the particular choice of graph representation can
fundamentally determine which dependencies are exposed to the learner and which are
obscured, amplifying or suppressing the effectiveness of GNNs depending on how well the
induced structure aligns with the target task~\cite{cucumides2025augraph}. Despite these
insights, the choice of graph construction in much of this literature is still treated primarily
as a modeling or engineering decision, with emphasis placed on empirical performance
rather than on formally characterizing the representational consequences of lifting tables to
graphs. In contrast, we do not propose new graph learning methods or graph constructions;
instead, we provide a formalization of the table-to-graph process and introduce a framework
and vocabulary for studying expressiveness, clarifying how different graph constructions
delimit the class of tabular prediction functions that can be represented.

\paragraph{Tabular models on graph-structured benchmarks.}
Recent work has shown that strong tabular learners can be competitive on graph benchmarks
when augmented with explicit feature constructions that expose local relational or
neighborhood information, such as fixed aggregations or structural
summaries~\citep{bazhenov2025GraphLand, FaF}. These results suggest that much of the
empirical gap between tabular and graph-based models can be attributed to access to
inter-node or inter-row information---often shallow---rather than to architectural differences
alone, with tabular models benefiting from their robustness, stable optimization, and ability
to handle heterogeneous features once such information is made explicit.

This perspective is related to our case study, but differs from our work in both object and
objective. Prior approaches based on neighbor feature aggregation or fixed aggregation
features are developed for graph datasets and are primarily used as diagnostic tools for
graph benchmarks, where tabular learners serve as ``must-beat'' baselines~\cite{relbench_neurips,
bazhenov2025GraphLand}: when they match or outperform graph neural networks, the
conclusion is that the benchmark does not structurally require learned message passing. By
contrast, our focus is on tabular datasets, where relational structure is implicit rather than
given. Instead of using tabular models to evaluate the adequacy of graph benchmarks, we
study the structural limitations of row-local tabular learning itself, and show that explicitly
exposing inter-row structure can complement strong tabular representations even in purely
tabular settings.

\paragraph{Graph-to-tabular foundation models.} A recent line of work --- including G2T-FM~\citep{g2t}, TabPFN-GN~\citep{graphpfn}, and TAG-style approaches --- adapts tabular foundation models to graph-structured data by tabularizing graph neighbourhoods through fixed structural features (often NFA-style aggregations) and feeding the result into a strong tabular learner such as TabPFN. Our hybrid in Section~\ref{sec:relbench-trial} resembles these approaches superficially: both feed structural information to a tabular learner. The contributions are different. G2T-FM-style methods propose a specific tabularization recipe optimized for transferring tabular foundation models to graph data. Our framework characterizes \emph{when} a tabularization is sufficient and when it is not: extension-invariant targets require none, shallow inter-row signals admit NFA-style compilation, and compositional targets such as DIAMOND require feature width that grows exponentially with the depth of the pattern (Section~\ref{sec:nfa-mpnn}). G2T-FM is one instance in this hierarchy --- a constructor that maps graphs to row features --- and the framework predicts the regime in which such an instance suffices. GraphPFN, which combines PFN-style global attention with graph message passing, is a different point in the hierarchy: a learned predictor over an explicit graph constructor, with a richer cross-row interaction than NFA but operating under the protocol of in-context learning. The framework places these methods on a common axis rather than competing with them.

%% file: appendix/appendix_tabular-nfa.tex
\section{Time-Aware NFA}
\label{app:tab-ta-nfa}

We describe a \emph{time-aware} variant of Neighbourhood Feature Aggregation (NFA), which is a
specialization of the construction introduced in Appendix~\ref{app:nfa_grable}.
The purpose of this variant is to enforce temporal causality in transactional and sequential
settings, while preserving the representation format and notation of NFA.

\begin{example}[Time-aware NFA grable]
Fix a base constructor $\gamma$ and let
\[
G^\gamma_{C,T}=(V,\;E_1,\dots,E_m,\;\rho)
\]
be its grable on $(C,T)$ with row nodes
$V_{\mathrm{row}}=\{v_r\mid r\in T\}$, as in Appendix~\ref{app:nfa_grable}.
Assume that each row $r\in T$ is associated with a timestamp $t(r)\in\mathbb{R}$,
stored as one of the row attributes.

Write $E=\bigcup_{i=1}^m E_i$ and $N(v)=\{u\in V_{\mathrm{row}}\mid (v,u)\in E\}$.
For a fixed time window $W>0$, define the \emph{time-aware neighborhood} of a row node $v\in V_{\mathrm{row}}$ as
\[
N_W(v)
\;=\;
\{\, u\in N(v)\mid t(u)<t(v)\ \wedge\ t(v)-t(u)\le W \,\}.
\]

The time-aware NFA augmentation for a row node $v\in V_{\mathrm{row}}$ is defined exactly as in
Appendix~\ref{app:nfa_grable}, except that all aggregations are computed over $N_W(v)$ instead of $N(v)$.
Specifically, for any scalar feature coordinate $z$ (numerical, or a one-hot indicator derived from
a categorical feature), let
\[
S^{(W)}_z(v)
=
\{\, z(u)\mid u\in N_W(v),\ z(u)\neq \mathsf{NaN}\,\}.
\]
Then $\mathrm{NFA}_W(v)$ concatenates:
(i) for each numerical coordinate $z$:
$\mathrm{mean}(S^{(W)}_z(v))\Vert \min(S^{(W)}_z(v))\Vert \max(S^{(W)}_z(v))$;
(ii) for each one-hot indicator $z$ of each categorical feature:
$\mathrm{mean}(S^{(W)}_z(v))$.

The enriched row features are given by
\[
\rho^{\mathrm{NFA}(W)}(v)
\;=\;
\rho(v)\ \Vert\ \mathrm{NFA}_W(v)
\qquad (v\in V_{\mathrm{row}}).
\]
Finally, the \emph{time-aware NFA grable} is the \emph{trivial} grable on row nodes with these features:
\[
G^{\mathrm{NFA}(W)}_{C,T}
\;=\;
\bigl(V_{\mathrm{row}},\;\emptyset,\;\rho^{\mathrm{NFA}(W)}\bigr).
\]
We write $\gamma_{\mathrm{NFA}(W)}(\gamma)$ for the resulting constructor.
\hfill$\blacktriangleleft$
\end{example}

\paragraph{Discussion.}
This construction does not introduce a new representation class.
It is a time-restricted instance of NFA, obtained by replacing the neighborhood
$N(v)$ in Appendix~\ref{app:nfa_grable} with its temporally filtered variant $N_W(v)$.
All expressiveness considerations for NFA therefore apply unchanged.

%% file: appendix/appendix-datasets.tex
\section{Experimental setup}

All experiments were run on Google Colab Pro. GPU backends varied with runtime availability and included NVIDIA T4, V100, and A100 cards; we did not control for backend across runs, as none of our experiments are compute-bound at this scale. For the larger RelBench datasets (\texttt{relbench-trial} in Section~\ref{sec:experiments} and the additional tasks in Appendix~\ref{app:relbench_additional}), we used Colab's high-RAM runtime to accommodate the full relational entity graph in memory. Synthetic and retail experiments ran comfortably on the standard runtime.

Tabular models (LightGBM, RealMLP) are tuned via AutoGluon with a 5-minute compute budget per task in medium-quality mode; relaxing this budget did not improve performance in preliminary trials. GNN training (HeteroSAGE, full-batch) completes within one hour per configuration. For TabPFN-2.5, we use the official TabPFN client\footnote{\url{https://github.com/PriorLabs/tabpfn-client}}, which delegates inference and finetuning to Prior Labs' hosted backend; no local GPU is used for TabPFN runs. The full set of reported experiments, including 5-run averaging and adjustments across configurations, is reproducible within roughly 40--60 GPU-hours on the hardware described above, excluding TabPFN client calls.

\subsection{Datasets} \label{app:datasets}
\subsubsection{Synthetic transactions dataset}
We construct a synthetic transaction dataset to study learning scenarios in which prediction targets depend on relationships between rows, while keeping the data representation itself purely tabular. Each row corresponds to a single transaction and contains only standard tabular attributes; no explicit relational or graph structure is encoded in the dataset.

\paragraph{Tabular schema.}
Each transaction is represented by the following columns:
\begin{itemize}
    \item \texttt{id}: a unique, sequential transaction identifier.
    \item \texttt{card\_id}: a categorical identifier representing the payment card used.
    \item \texttt{merchant\_id}: a categorical identifier representing the merchant.
    \item \texttt{merchant\_city}: a categorical attribute indicating the merchant’s city, or the special value \texttt{ONLINE} for online merchants.
\end{itemize}

\paragraph{Merchant generation.}
A fixed pool of merchants is generated for each dataset split. Each merchant is assigned a unique integer identifier. Merchants are independently labeled as \emph{online} with a fixed probability (the \emph{online share}); online merchants are assigned the city label \texttt{ONLINE}. All remaining merchants are considered physical and are assigned a city sampled uniformly at random from a predefined list of European cities. The merchant city attribute is therefore a deterministic function of the merchant identifier.

\paragraph{Card generation and coverage.}
Card identifiers are generated as formatted strings (e.g., \texttt{C000317}) to mimic realistic identifiers and avoid numerical formatting issues. To prevent degenerate cases where a card never appears, the data generation process enforces that \emph{every card identifier occurs at least once}. This is achieved by assigning each card to exactly one transaction before any additional sampling takes place.

\paragraph{Transaction-level sampling.}
After guaranteeing one occurrence per card, the remaining transactions are populated by sampling card identifiers with replacement from the card pool. Sampling probabilities follow a bounded power-law distribution, so that some cards appear more frequently than others. This induces heterogeneous card usage frequencies while maintaining full coverage.

Merchant identifiers are sampled independently for each transaction, also using a skewed distribution. As a result, a small number of merchants account for a disproportionate number of transactions, while most merchants appear infrequently. This design reflects common properties of real transaction data and is critical for inducing non-trivial inter-row dependencies.

\paragraph{Assembly and randomization.}
Transaction rows are assembled by combining transaction identifiers, sampled card identifiers, and sampled merchant identifiers, and then joining with the merchant table to attach the corresponding city attribute. The resulting table is randomly shuffled to remove any ordering artifacts introduced by the card-coverage step. After shuffling, transaction identifiers are reassigned to ensure they remain sequential.

\paragraph{Train, validation, and test splits.}
Training, validation, and test datasets are generated independently using the same procedure but with different random seeds and dataset sizes. No identifiers are shared across splits, ensuring an inductive evaluation setting. The exact dataset sizes and parameters are summarized in Table~\ref{tab:synthetic-transactions-stats}.

\paragraph{Example rows.}
Table~\ref{tab:synthetic-transactions-example} shows representative rows from a generated dataset. Card and merchant identifiers repeat across rows, and the merchant city attribute is entirely determined by the merchant identifier.

\begin{table}[h]
\centering
\caption{Example rows from the synthetic transaction dataset.}
\vspace{1em}
\label{tab:synthetic-transactions-example}
\begin{tabular}{r l r l}
\hline
\textbf{id} & \textbf{card\_id} & \textbf{merchant\_id} & \textbf{merchant\_city} \\
\hline
1 & C000317 & 42  & Brussels \\
2 & C000102 & 17  & ONLINE   \\
3 & C000317 & 105 & Paris    \\
4 & C000891 & 42  & Brussels \\
5 & C000045 & 311 & Berlin   \\
6 & C000102 & 17  & ONLINE   \\
\hline
\end{tabular}
\end{table}

\paragraph{Dataset statistics.}
Table~\ref{tab:synthetic-transactions-stats} reports summary statistics for each split. By construction, every card appears at least once in each split, while both card usage and merchant usage exhibit heavy-tailed frequency distributions.

\begin{table}[h]
\centering
\caption{Summary statistics of the synthetic transaction datasets.}
\vspace{1em}
\label{tab:synthetic-transactions-stats}
\begin{tabular}{l r r r}
\hline
\textbf{Statistic} & \textbf{Train} & \textbf{Validation} & \textbf{Test} \\
\hline
\# transactions        & 8{,}000 & 1{,}000 & 1{,}000 \\
\#  unique cards        & 2{,}500 & 350     & 350     \\
\# unique merchants    & 3{,}500 & 300     & 300     \\
Online merchants     & 0.15    & 0.12    & 0.12 \\
\hline
\end{tabular}
\end{table}

This generation process deliberately separates \emph{representation} from \emph{structure}: although the data are provided as a simple table, repeated categorical values induce implicit cross-row dependencies. This allows us to precisely control which prediction tasks require reasoning beyond independent rows, while avoiding confounds introduced by explicit relational encodings.

\subsubsection{Real transactions dataset (Retail)}

We use a transactional dataset containing all purchases recorded between 01/12/2010 and 09/12/2011 for a UK-based, registered non-store online retailer.
The dataset consists of 541,909 transaction records and exhibits multivariate, sequential, and time-series characteristics.
Each row corresponds to an individual transaction line item associated with a specific invoice, customer, and timestamp.
The retailer primarily sells unique, all-occasion gifts, and a substantial fraction of its customers are wholesalers.
For our experiments, we restrict attention to the final 10\% of the data in chronological order.
This subset is further split based on timestamps into training, validation, and test sets using an 80/10/10 temporal split, ensuring strict temporal consistency and preventing information leakage across splits.

Table~\ref{tab:transaction_example} shows a simplified example of transaction rows from the dataset.
Each row represents a single purchased item within an invoice; multiple rows may share the same invoice number, customer identifier, and timestamp.

\begin{table}[h]
\centering
\small
\caption{Illustrative example of transaction rows (synthetic but representative). Multiple rows correspond to distinct items within the same invoice.}
\vspace{1em}
\begin{tabular}{l l l r l r l l}
\hline
InvoiceNo & StockCode & Description & Qty & InvoiceDate & Price & CustomerID & Country \\
\hline
581220 & 16225 & Rattle Snake Eggs & 24 & 2011-12-08 & 0.39 & 13924 & UK \\
581220 & 16258A & Circular Rubbers & 36 & 2011-12-08 & 0.19 & 13924 & UK \\
581220 & 21411 & Heart Doorstop (Red) & 8 & 2011-12-08 & 1.95 & 13924 & UK \\
\hline
\end{tabular}

\label{tab:transaction_example}
\end{table}

\subsubsection{\texttt{relbench-trial}}
The \texttt{relbench-trial} dataset is a relational clinical trial database derived from the Aggregate Analysis of ClinicalTrials.gov (AACT) initiative, which consolidates protocol and results information for studies registered on ClinicalTrials.gov. The dataset contains structured metadata describing clinical studies, including study design characteristics, interventions, conditions, outcomes, eligibility criteria, sponsors, facilities, and administrative information.

The dataset is organized as a multi-table relational database consisting of 15 tables. The primary table for prediction tasks is \texttt{studies}, where each row corresponds to a single registered clinical trial. This table includes high-level descriptive attributes such as study type, phase, enrollment information, allocation, masking, intervention model, sponsor details, and study status. The prediction target is defined at the level of the \texttt{studies} table following the RelBench task specification.

The remaining tables provide auxiliary information linked to \texttt{studies} through primary--foreign key relationships. These tables describe clinical conditions, interventions, outcomes, design groups, sponsors, facilities, and eligibility criteria, among others. Together, these tables define a heterogeneous relational schema, which is illustrated in Figure~\ref{fig:trial-schema}.

The temporal coverage of the dataset begins on January 1, 2000. The dataset follows the RelBench benchmark protocol, which defines time-based splits and prediction windows.

\begin{table}[h]
\centering
\caption{Summary statistics of the \texttt{relbench-trial} dataset.}
\vspace{1em}
\label{tab:relbench-trial-stats}
\begin{tabular}{ll}
\toprule
\textbf{Attribute} & \textbf{Value} \\
\midrule
Domain & Medical \\
Number of tables & 15 \\
Number of rows & 5,852,157 \\
Number of columns & 140 \\
Start date & 2000-01-01 \\
Validation timestamp & 2020-01-01 \\
Test timestamp & 2021-01-01 \\
Prediction time window & 1 year \\
\bottomrule
\end{tabular}
\end{table}
\begin{figure*}[t]
    \centering
    \includegraphics[width=0.9\linewidth]{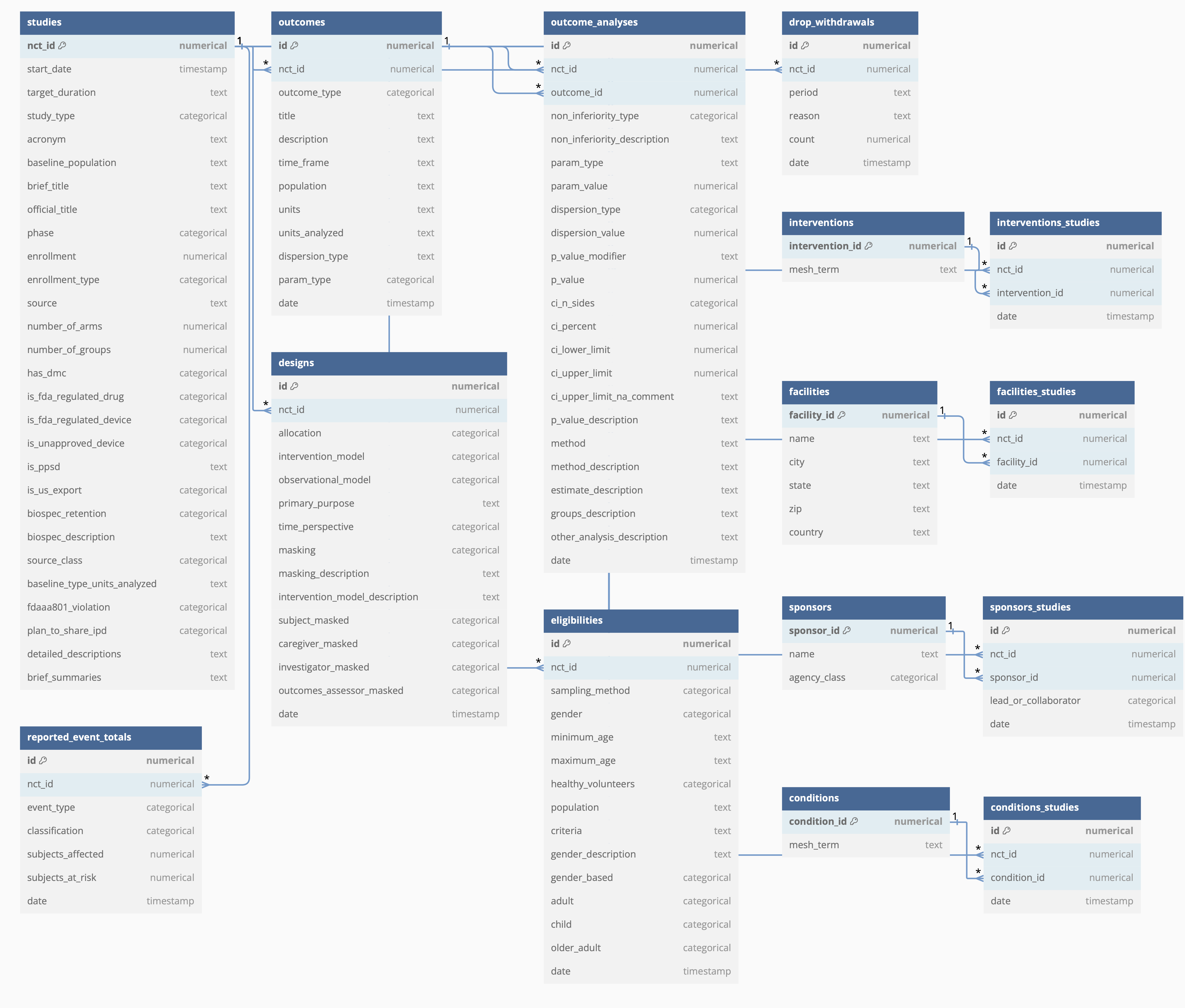}
    \caption{Relational schema of the \texttt{relbench-trial} relational dataset. The task studied is associated with table \texttt{studies} that contain descriptive information about the clinical studies. Taken from \url{https://relbench.stanford.edu/datasets/rel-trial/}}
    \label{fig:trial-schema}
\end{figure*}

\subsection{Logical tasks generation} \label{app:task-generation}

To evaluate models on controlled inter-row dependencies, we derive binary classification labels from a given transaction table by applying four deterministic labeling functions. Each function maps the input table to an augmented table by adding one new label column. Importantly, all labels are defined \emph{within a single split} (train/validation/test) using only the rows in that split; the label of a row therefore depends on the multiset of values present in the corresponding split, rather than on any external data.

Throughout, let $D$ denote a table with rows indexed by $i \in \{1,\dots,|D|\}$ and columns (attributes) such as \texttt{card\_id}, \texttt{merchant\_id}, and \texttt{merchant\_city}. Each task produces a label $y_i \in \{0,1\}$ for each row $i$.

\paragraph{Task 1: \textsc{UNIQUE}}
Given a column $c$, the uniqueness task labels a row as positive if its value in column $c$ occurs \emph{exactly once} in the entire table. 
In the implementation, we compute the frequency (value count) of each distinct value in column $c$ and set the label to 1 precisely for rows whose value has frequency 1.

\paragraph{Task 2: \textsc{COUNT} and \textsc{COUNT}(=).}
Given a column $c$ and an integer threshold $k$, the counting task labels a row based on how often its value appears in column $c$ across the table. Two variants are supported:

\begin{itemize}
    \item \textbf{Greater-than variant:} $y_i = 1$ iff the value in column $c$ appears more than $k$ times. 
    \item \textbf{Equality variant:} $y_i = 1$ iff the value in column $c$ appears exactly $k$ times. 
\end{itemize}

Operationally, both variants are implemented by computing value counts for column $c$ and mapping the corresponding count back to each row before applying the threshold test.

\paragraph{Task 3: \textsc{DOUBLE} (existence of another matching row with an anchor attribute).}
Given two columns $c_1$ and $c_2$, and a designated anchor value $a$ in column $c_2$, the double-condition task labels a row as positive if there exists \emph{another} row that shares the same $c_1$ value and whose $c_2$ value equals the anchor. 
The implementation computes, for each distinct value $v$ of $c_1$, how many rows satisfy $(c_1=v) \wedge (c_2=a)$. This count is then mapped back to each row. To enforce the ``other row'' requirement, the labeling distinguishes two cases:
\begin{itemize}
    \item If $D[i,c_2] \neq a$, then $y_i=1$ iff there is at least one anchor row with the same $c_1$ value.
    \item If $D[i,c_2] = a$, then $y_i=1$ iff there are at least two anchor rows with the same $c_1$ value (so that one of them can serve as the required distinct witness row).
\end{itemize}

\paragraph{Task 4: \textsc{DIAMOND} (duplicate pair).}
Given two columns $c_1$ and $c_2$, the diamond task (implemented as duplicate-pair detection) labels a row as positive if its $(c_1,c_2)$ pair occurs in more than one row. 
In the implementation, we group by the pair $(c_1,c_2)$, compute group sizes, and assign label 1 to all rows belonging to groups of size greater than 1. 

\paragraph{Summary.}
All four labeling functions are deterministic and depend on global frequency or co-occurrence properties of the table. Consequently, labels are \emph{extension-sensitive}: duplicating, removing, or adding rows can change the label of an existing row. This property is central to our experimental objective of isolating tasks that cannot be represented by strictly row-local predictors without access to inter-row information.

Table \ref{tab:proportions_tasks} shows the resulting statistics of the task labels in both datasets for the different splits.
\input{tables/tab_proportions-tasks-datssets}

\subsection{Inference protocols and extension-sensitive targets}
\label{app:inference_protocol}

The grables framework supports three inference protocols (Section~\ref{sec:inference-protocols}): single-point, single-point-with-reference, and set-based. Set-based inference is the standard protocol in graph learning, relational deep learning, and inter-sample transformer-based tabular models~\citep{relbench_neurips, kossen-npt, tabpfn_nature}, but its interaction with extension-sensitive targets is worth documenting explicitly. For extension-invariant targets, all three protocols are appropriate and yield equivalent evaluation. For extension-sensitive targets, the choice of protocol determines whether the relevant context is exposed at inference time. We document this here for the \textsc{COUNT} task as a representative case; analogous reasoning applies to the other extension-sensitive targets considered in the paper.

\paragraph{Setup.} Recall that \textsc{COUNT}$_N$ labels a row $r$ as positive if at least $N$ other rows in the table share a designated value with $r$ (Section~\ref{sec:incidence-separations}). Under set-based inference, the ``table'' is the full evaluation split $T$, and the label of $r$ is computed against $T \setminus \{r\}$. Under single-point-with-reference, the table is $R \cup \{r\}$ where $R$ is a fixed reference set (typically the training table), and only $r$ is scored.

\paragraph{Memorization shortcut under single-point-with-reference.} For a test row $r$ with designated value $v$, the label of $r$ under this protocol is essentially determined by how often $v$ already appears in $R$. Concretely, test rows fall into three cases:

\begin{itemize}
    \item \textbf{Case A} ($v$ appears at least $N$ times in $R$): label is $1$. The same value $v$ already had label $1$ in the training data, so memorizing ``rows with value $v$ are positive'' suffices.
    \item \textbf{Case B} ($v$ appears exactly $N - 1$ times in $R$): the label flips from $0$ to $1$ when $r$ is added. This is the only case where extension-sensitivity actually binds within the protocol.
    \item \textbf{Case C} ($v$ appears fewer than $N - 1$ times in $R$): label is $0$. The same value $v$ had label $0$ in the training data, and memorizing this suffices.
\end{itemize}

In Cases A and C, which together account for the vast majority of test rows, a model can predict correctly by memorizing the mapping from feature values to labels seen during training: rows whose value was associated with label $1$ in $R$ get predicted as $1$, and similarly for $0$. No inter-row reasoning is required, and no statistic about $R$ needs to be computed at inference time — the memorized value-to-label association is itself a row-local feature. Only Case B genuinely requires the model to recognize that adding $r$ tips the count over the threshold, which is the property the framework is designed to study.

As a consequence, aggregate evaluation metrics under this protocol are dominated by performance on Cases A and C, where any model with sufficient capacity to memorize value-label associations on $R$ can score well — regardless of whether it performs genuine inter-row reasoning.

\paragraph{Empirical illustration.} Table~\ref{tab:protocol_collapse} reports performance on COUNT$_{>5}$ on the synthetic transactions data under single-point-with-reference, broken down by case. Aggregate scores are similar across model classes, including row-local baselines, because Cases A and C dominate. On the Case B subset --- where extension-sensitivity actually binds --- the picture differs, but the subset is too small to drive aggregate metrics.

\begin{table}[h]
\centering
\caption{Performance on COUNT$_{>5}$ under single-point-with-reference protocol, decomposed by case. Aggregate metrics conflate memorization of value-label associations with extension-sensitive reasoning. Results averaged over 5 runs.}
\vspace{1em}
\label{tab:protocol_collapse}
\begin{tabular}{lcccc}
\toprule
Model & Aggregate AUC & Cases A and C AUC & Case B AUC \\
\midrule
LightGBM & 0.667 & 0.997& 0.497 \\
TabPFN & 0.691 & 0.869 & 0.501 \\
ModernNCA & 0.662 & 0.961 & 0.495 \\
\bottomrule
\end{tabular}
\end{table}

\paragraph{Implication for protocol choice.} The above analysis explains why we evaluate extension-sensitive targets using a full disjoint table as reference for inference in the main experiments (Section~\ref{sec:experiments}). This is a requirement for the evaluation to be informative: by using the training table as refence, COUNT-style targets reduce to memorization of value-label associations seen in $R$, and the constructor hierarchy of Section~\ref{sec:separations} is no longer binding on all but a small boundary subset of test rows.

This does not mean the framework is incompatible with the standard tabular inference protocol. For extension-invariant targets, single-point and single-point-with-reference protocols are appropriate, and the framework applies directly. The protocol-target interaction documented here is specific to extension-sensitive targets, and it sharpens the framework's central claim: the constructor hierarchy expresses representational reach, but realising that reach in practice requires both an appropriate constructor and an inference protocol that exposes the relevant context.

\subsection{Tabular models settings} \label{app:tabular_details}

To fit tabular models LightGBM and RealMLP we use the AutoML library AutoGluon. Performance report of Table \ref{tab:logical-tasks} corresponds to adjusting for 5 minutes in \texttt{medium} quality setting. Relaxing this constraint did not improve performance. 

For ModernNCA \citep{modern-nca}, we use the official implementation from the TALENT benchmark \citep{ye2024TALENT}. We use the default hyperparameters provided by TALENT, with \texttt{normalization=quantile}, \texttt{num\_policy=none}, and \texttt{cat\_policy=tabr\_ohe}, as required by the model. Early stopping is applied with a patience of 20 epochs on the validation set.

\subsection{TabPFN settings and finetuning} \label{app:tabpfn-settings}

\input{tables/tab_tabpfn-ft}
We use TabPFN2.5 \citep{tabpfn2.5}. We evaluate the tasks with and without finetuning. In general, as seen in Table \ref{tab:tabpfn_finetune} the finetuning effect is not consistent for all tasks and in some cases it degrades model performance. 

\subsection{Graph learning setting and models. } \label{app:gnn_details}

\subsubsection{Synthetic data setup} 
\paragraph{Graph construction.}
For synthetic experiments, we represent the data as a heterogeneous graph composed of a single \emph{transaction} node type (denoted \texttt{row}) and multiple \emph{categorical value} node types corresponding to discrete attributes (e.g., card identifiers, merchant identifiers, locations). Each transaction node is connected by a directed edge to exactly one node of each categorical type. Reverse edges are added for all relations, enabling information to propagate back to transaction nodes.

\paragraph{Node initialization.}
Transaction nodes are initialized with either constant features or a learned type embedding. Categorical value nodes do not carry attributes and are initialized structurally as the sum of:
(i) a learnable node-type embedding, and
(ii) a fixed, deterministic random hash vector.
The fixed hash vectors are non-trainable and serve to break symmetry between nodes with identical initial features and neighborhoods, which is particularly important in inductive settings.

\paragraph{Message passing.}
We use a multi-layer heterogeneous GraphSAGE architecture. For each relation type, messages are computed using a GraphSAGE convolution with \emph{sum aggregation}, preserving degree and count information. Messages from different relations are summed at each destination node. All message passing is performed in full-batch mode to avoid distortions of node degrees. Each layer applies a ReLU nonlinearity and optional dropout.

\paragraph{Prediction head and training.}
After the final message-passing layer, transaction node embeddings are passed through a shallow multilayer perceptron to produce a binary logit. Models are trained using a weighted binary cross-entropy loss to account for class imbalance. Decision thresholds are selected on a validation set to maximise the F1 score.

This model serves as a clean structural baseline, designed to evaluate the representational capacity of message-passing GNNs for logical and counting predicates under idealized conditions.

\paragraph{Training hyperparameters. } Table \ref{tab:hyperparams_synthetic} displays the training hyperparameters per task. 

\input{tables/tab_synth-training-hyperparameters}

\subsubsection{Real transaction data setup}
\label{app:real-gnn}

\paragraph{Graph construction.}
For real transaction data, we use the same heterogeneous graph schema but including \emph{all} available features (e.g., card, client, merchant, city, state). This setting reflects realistic relational noise, where multiple relations coexist and only a subset may be relevant to a given prediction task. As in the synthetic setting, reverse edges are added for all relations.

Separate graphs are constructed for training, validation, and testing. Node identities are not shared across splits, ensuring an inductive evaluation protocol.

\paragraph{Node initialization.}
Transaction nodes are initialized from real-valued feature vectors derived from tabular attributes (e.g., transaction amount, time features, and hashed categorical indicators). These features are linearly projected into the hidden embedding space.

Feature-value nodes do not carry attributes and are initialized structurally using the same mechanism as in the synthetic data setting: a learnable type embedding plus a fixed deterministic hash vector. This ensures symmetry breaking without relying on shared node identities across splits.

\paragraph{Message passing and training.}
Both the model architecture and training hyperparameters are kept constant from the synthetic data case. 






\subsubsection{Case study data setup}
For the case study, we distinguish three models. 
\begin{enumerate}
    \item \textbf{Tab.} A tree-based row-local (potentially ensemble) method trained using AutoGluon over all of its tree-based methods. Best model is chosen using validation ROC-AUC
    \item \textbf{GNN on tabular incidence graph.} HeteroSage architecture trained on the incidence graph. 
    \item \textbf{GNN on relational entity graph.} To allow fair comparison, we use the GNN proposed by \citet{relbench_neurips} with the same training hyperparameters. 
\end{enumerate}

The tree-based model can either be fed with the original tabular features (e.g. Tree(T)), or with the original features plus GNN-learned embeddings (which range from 64 to 128 new added dimensions).

\section{More experimental results and analysis}

\subsection{Full version of Table \ref{tab:logical-tasks}}
\label{app:full-table-logical-tasks}

Tables \ref{tab:logical-tasks-1} and \ref{tab:logical-tasks-2} show a full version of Table \ref{tab:logical-tasks} including standard deviations over the five runs. 

\input{neurips/neurips-tables-figures/table-logical-std}
\subsection{Synthetic data} 
\subsubsection{Sensitivity Analysis} \label{app:synthetic-sensitivity}
To assess the consistency of the models in the logical task, we perform a sensitivity analysis where we perturb the test set (while keeping the distribution parameter constant). 

Figure \ref{fig:sensitivity-f1} shows the variation of F1-score of models LightGBM, RealMLP, TabPFN and Sage in 5 different perturbations of the test set. 

\begin{figure}[h]
    \centering
    \includegraphics[width=0.9\linewidth]{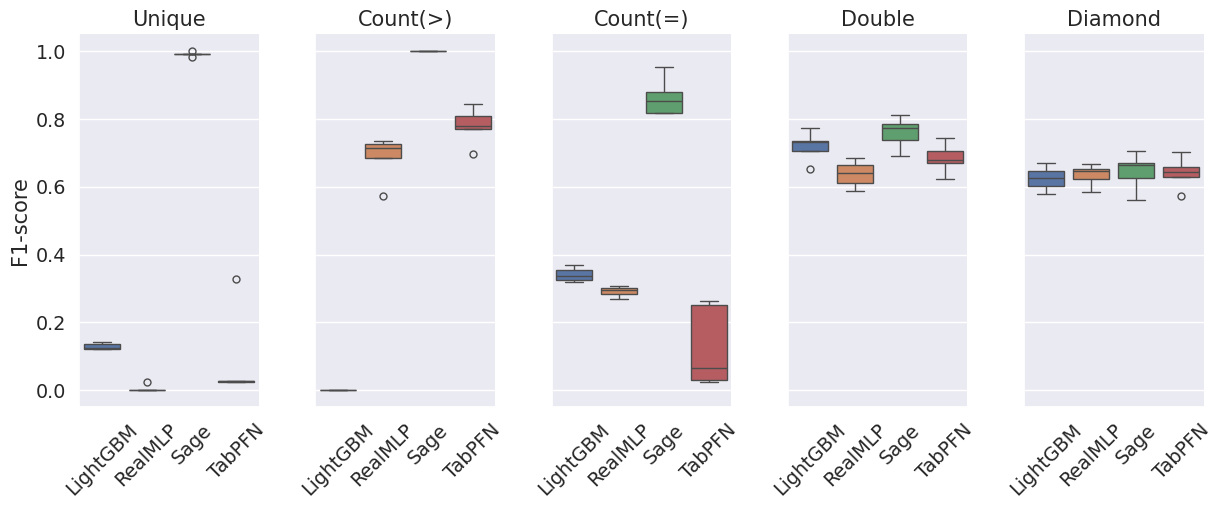}
    \caption{Box-plots for F1 score of models LightGBM, RealMLP, Sage and TabPFN in 5 different perturbations of the original test set. }
    \label{fig:sensitivity-f1}
\end{figure}

\subsubsection{Stress-test set} \label{app:stress-set}

On the original test splits, row-local models reach high scores on tasks they provably cannot represent (Table~\ref{tab:logical-tasks}). Such scores are possible only if the labels correlate with row-local features. The stress-test set removes these correlations, so that a predictor that depends only on the row cannot score above chance, and a high score requires information from other rows. Section~\ref{sec:stress-test} uses it to test whether the scores in Table~\ref{tab:logical-tasks} come from inter-row structure or from row-local shortcuts.

\paragraph{Stress-test set generation.} For each task, we first check which row-local features correlate with the label on the original splits, and then generate a test set in which these features do not predict the label:
\begin{itemize}
    \item \textsc{Unique}. We found no row-local shortcut on the original splits. The stress-test set inverts the labels relative to training: values that are unique in the training split occur more than once in the stress-test set, so a predictor based on value identity in training predicts the wrong label.
    \item \textsc{Count}$(>)$. The most frequent values in training are also the most frequent in validation and test, so value identity is a proxy for the label. In the stress-test set, values with a high count in training occur fewer times than the threshold.
    \item \textsc{Count}$(=)$. We found no row-local proxy on the original splits. The stress-test set is generated with the same procedure as the original data (Appendix~\ref{app:datasets}), with a different random seed.
    \item \textsc{Double}. The label can be approximated by a single condition, in particular the more frequent of the two, the anchor value $c_2 = a$. In the stress-test set, neither condition alone predicts the label.
    \item \textsc{Diamond}. The label can be approximated by the most frequent values of a single column, since frequent values are more likely to form a repeated pair. The stress-test set is sampled with balanced value frequencies, so the most frequent value of a column in training does not predict a repeated pair.
\end{itemize}

\paragraph{Results on stress-test set.} Models are trained on the original training split and evaluated on the stress-test set without retraining. Results are shown in Figure~\ref{fig:stress-set} and discussed in Section~\ref{sec:stress-test}.

\subsection{Transactional data} 
\subsubsection{Removing distracting information.} \label{app:transactions-remove-noise}

To address how the extra information coming from the features that are not considered in the logical tasks affects the learning and performance of the different models, we create a simplified version of the dataset, where only features used in the creation of the logical tasks are considered. This mirrors the synthetic dataset properties. 

\input{tables/tab_retail-full-vs-clean}

Tables \ref{tab:appendix_retail_part1} and \ref{tab:appendix_retail_part2} show both F1 and ROC-AUC performance for the 4 logical tasks, using the original retail transactions data and the simplified (or clean) version of it. We observe that the GNN maintains its performance for almost every task, with the exception of Counting with equality, where it lifts significantly. We also observe that RealMLP has a very good performance on the \textsc{Double} task, however we attribute it to shortcut learning of a single condition instead of both, as this result does not survive the stress-test done in \ref{app:stress-set}.


\subsection{Case study} 
\subsubsection{Feature importance of different models} \label{app:feature-importance}

\paragraph{Tabular only model.} We train a tabular-only model using AutoGluon’s tree-based ensemble methods and evaluate feature importance via permutation importance with respect to the ROC-AUC score. As shown in Figure~\ref{fig:tab-feature-importance}, text-derived fields dominate the model’s predictive performance, while most structured metadata features exhibit smaller or statistically indistinguishable effects. Several features yield near-zero or negative importance, indicating limited or potentially detrimental contribution to model performance.

\begin{figure}[h]
    \centering
    \includegraphics[width=0.7\linewidth]{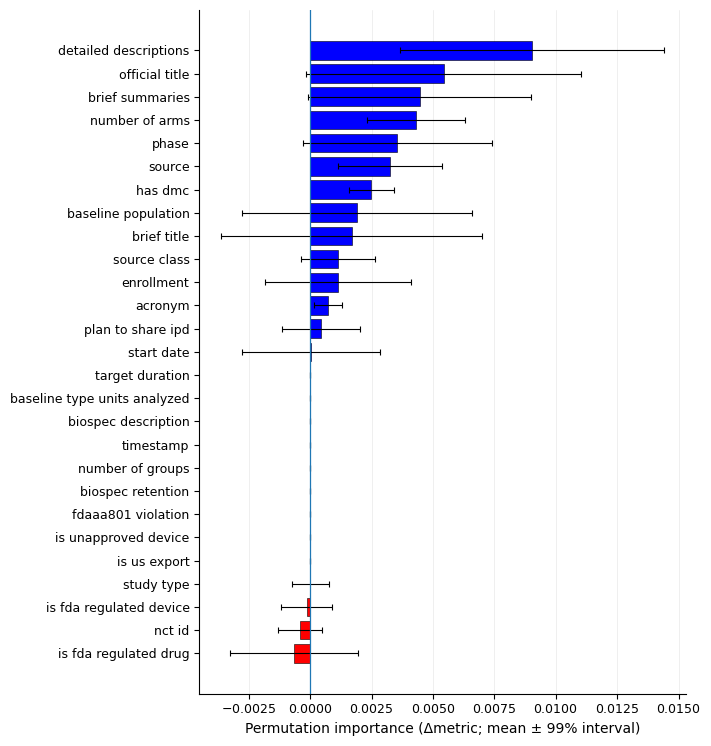}
    \caption{
Permutation feature importance based on ROC-AUC score for the AutoGluon tabular model.
Bars show the mean change in ROC-AUC score under feature permutation, with 99\% confidence intervals from repeated shuffles.
Positive values indicate performance degradation, while negative values indicate improved performance after permutation.
}
    \label{fig:tab-feature-importance}
\end{figure}

\paragraph{Tabular + NFA (intra-table).} A similar tree-based tabular model, trained and ensembled with Autogluon is applied over the tabular data with extra features coming from the time-aware tabular neighbor feature aggregation procedure (Appendix \ref{app:tab-ta-nfa}). Figure \ref{fig:importance-nfa} shows the most important 15 features. The top features are shared with the tabular only method, but from the 15 most important features, 6 correspond to newly aggregated features. 

\begin{figure}[h]
    \centering
    \includegraphics[width=0.7\linewidth]{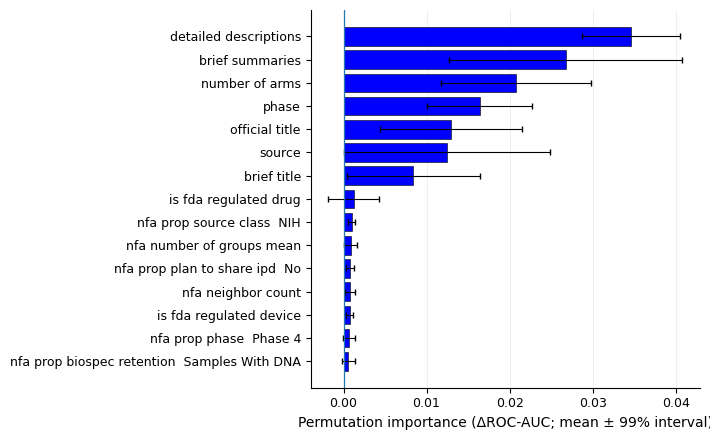}
    \caption{Permutation feature importance based on ROC-AUC score for the AutoGluon tabular model with tabular time-aware neighbor feature aggregation features. Bars show the mean change in ROC-AUC score ($\Delta$ROC-AUC) under feature permutation, with 99\% confidence intervals from repeated shuffles. Trimmed to most relevant 15 attributes. }
    \label{fig:importance-nfa}
\end{figure}

\paragraph{Tabular + GNN embeddings (intra-table).} Figure \ref{fig:importance-tab-gnn-t} shows the 15 most important features for the tabular tree-based model trained on the tabular data augmented with embeddings obtained from training a GNN over the incidence graph of the table. Out of the 15 most relevant features, 9 correspond to GNN embeddings generated features, showing how the model is actually using them to predict. At the top, there still remain the free text features from the data. 

\begin{figure}[h]
    \centering
    \includegraphics[width=0.7\linewidth]{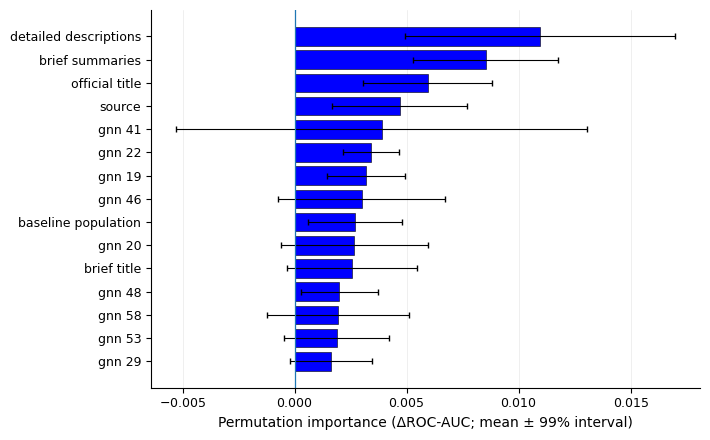}
    \caption{Permutation feature importance based on ROC-AUC score for the AutoGluon tabular model augmented with embeddings from a GNN trained on the incidence grable from the table. Bars show the mean change in ROC-AUC score ($\Delta$ROC-AUC) under feature permutation, with 99\% confidence intervals from repeated shuffles. Trimmed to most relevant 15 attributes. Important to note that out of the 15 most important attributes, 9 correspond to GNN embeddings. }
    \label{fig:importance-tab-gnn-t}
\end{figure}

\paragraph{Tabular + GNN embeddings (relational entity graph).} Figure \ref{fig:importance-tab-gnn-reg} shows the 15 most important features for the tabular tree-based model trained on the tabular data augmented with embeddings obtained from training a GNN over the relational entity graph derived from the full relational database (Figure \ref{fig:trial-schema}). Out of the 15 most relevant features, 7 correspond to GNN embeddings generated features. Moreover, it is interesting to see how more of the original tabular features place at the top when the intra-structure of the table is not exposed in the graph (and therefore not properly captured in the GNN embeddings). 

\begin{figure}[h]
    \centering
    \includegraphics[width=0.7\linewidth]{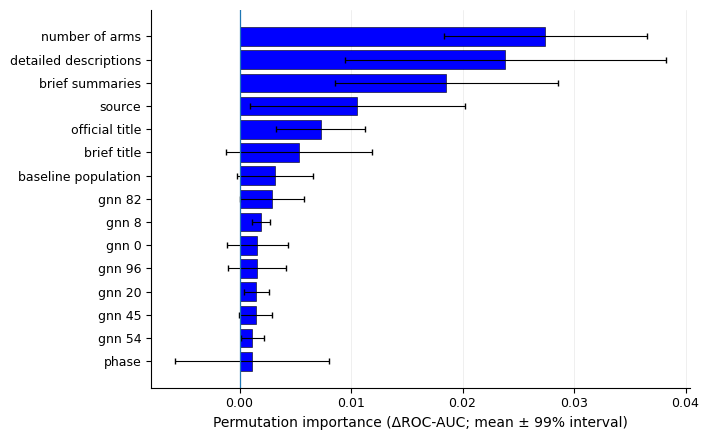}
    \caption{Permutation feature importance based on ROC-AUC score for the AutoGluon tabular model augmented with embeddings from a GNN trained on the relational entity graph from the entire database.
Bars show the mean change in ROC-AUC score ($\Delta$ROC-AUC) under feature permutation, with 99\% confidence intervals from repeated shuffles. Trimmed to most relevant 15 attributes. Important to note that out of the 15 most important attributes, 7 correspond to GNN embeddings. }
    \label{fig:importance-tab-gnn-reg}
\end{figure}


\subsection{Results on i.i.d.\ tabular benchmark TabArena}
\label{app:tabarena}

The framework predicts that gains from exposing inter-row structure should be negligible on datasets whose targets are genuinely extension-invariant. Standard i.i.d.\ benchmarks are constructed under assumptions that align with row-local learning---random splits, independently drawn labels---so any cross-row signal that exists is incidental rather than definitional. We nonetheless run experiments to verify this empirically and to surface any datasets where latent inter-row structure happens to be informative.

\paragraph{Dataset selection.}
From TabArena~\citep{erickson2025tabarena}, we select datasets with a binary target and at least 3 categorical features with non-unique values, as these are necessary conditions for inter-row dependencies to be exploitable via the incidence grable. This yields 9 datasets, ranging from 3 to 86 categorical features.

\paragraph{Setup.}
We compare a tabular baseline (LightGBM tuned via AutoGluon) against a hybrid that augments the same model with GNN embeddings obtained from message passing on the incidence graph---the same setup as in the \texttt{relbench-trial} case study (Section~\ref{sec:relbench-trial}). Results are averaged over 5 runs; we report mean $\pm$ std.

\paragraph{Results.}
Table~\ref{tab:tabarena} reports ROC-AUC for both configurations. Gains are small and
inconsistent: deltas are near zero for most datasets, with only three exceeding $2\%$---all
of which are positive. This is consistent with the theoretical prediction: when the label is
not structurally defined by inter-row relationships, exposing that structure provides no
reliable benefit.

\begin{table}[h]
\centering
\caption{ROC-AUC on 9 TabArena datasets selected for having a binary target and at least 3 categorical features with non-unique values. Results averaged over 5 runs (mean $\pm$ std).}
\vspace{1em}
\label{tab:tabarena}
\begin{tabular}{lcrrrr}
\toprule
Dataset ID & \# Cat.\ features & Tabular & Hybrid & $\Delta$ \\
\midrule
363619 &  4 & \ms{0.845}{.02} & \ms{0.848}{.02} & \textcolor{blue}{$+0.36\%$} \\
363623 &  4 & \ms{0.928}{.02} & \ms{0.924}{.01} & \textcolor{red}{$-0.43\%$} \\
363626 & 13 & \ms{0.717}{.01} & \ms{0.752}{.01} & \textcolor{blue}{$+4.88\%$} \\
363671 &  3 & \ms{0.784}{.01} & \ms{0.800}{.01} & \textcolor{blue}{$+2.04\%$} \\
363682 &  8 & \ms{0.680}{.02} & \ms{0.698}{.02} & \textcolor{blue}{$+2.65\%$} \\
363684 &  8 & \ms{0.890}{.01} & \ms{0.885}{.01} & \textcolor{red}{$-0.56\%$} \\
363689 & 86 & \ms{0.986}{.01} & \ms{0.984}{.01} & \textcolor{red}{$-0.20\%$} \\
363696 &  5 & \ms{0.893}{.01} & \ms{0.892}{.01} & \textcolor{red}{$-0.11\%$} \\
363700 &  3 & \ms{0.781}{.02} & \ms{0.777}{.02} & \textcolor{red}{$-0.51\%$} \\
\bottomrule
\end{tabular}
\end{table}

\paragraph{Takeaway.}
The results confirm the theoretical prediction: on datasets where the label is not structurally defined by inter-row relationships, structural augmentation provides no consistent benefit. This contrasts sharply with the controlled tasks in Section~\ref{sec:experiments}, where targets are definitionally extension-sensitive and the performance gap is large and consistent.

\subsection{Results on relational datasets}
\label{app:relbench_additional}

The \texttt{relbench-trial} case study in Section~\ref{sec:relbench-trial} was chosen precisely because it is an atypical task where the tabular baseline is competitive with graph-based models, making it the most informative setting for studying hybrid complementarity. To assess whether the hybrid gains generalise, we run experiments on 4 additional binary classification tasks spanning 3 new RelBench datasets.

\paragraph{Setup.}
We re-implement all baselines in a unified experimental setup rather than reporting numbers from the original papers. Absolute performance differs from published results, likely due to differences in text feature treatment, which is under-discussed in the RelBench literature. Since any such difference affects all models equally, relative comparisons between tabular, GNN, and hybrid variants remain valid. The hybrid setup follows Section~\ref{sec:relbench-trial}: GNN embeddings are concatenated with the original tabular features and fed into LightGBM. Results are averaged over 5 runs.

\paragraph{Results.}
Table~\ref{tab:relbench_additional} reports ROC-AUC across all tasks. Unlike \texttt{relbench-trial}, on most of these tasks the GNN trained on the relational entity graph already outperforms the row-local tabular baseline, confirming that the predictive signal is extension-sensitive. The exception is \texttt{driver-dnf}, where the tabular baseline marginally outperforms the GNN, suggesting that the signal on that task is more mixed. In all cases, the hybrid method outperforms both the purely tabular and the purely graph-based baseline, consistent with the framework's prediction that tabular and graph-based representations capture complementary information.

\begin{table}[h]
\centering
\caption{ROC-AUC on 5 additional RelBench binary classification tasks, re-implemented in a unified setup for fair comparison. Absolute numbers differ from published results due to text feature treatment; relative comparisons are unaffected. The hybrid matches or improves over both purely tabular and purely graph baselines. Results averaged over 5 runs (mean $\pm$ std).}
\vspace{1em}
\label{tab:relbench_additional}
\begin{tabular}{llccccc}
\toprule
Dataset & Task & Tabular & GNN & Hybrid & \% vs.\ Tab & \% vs.\ GNN \\
\midrule
f1    & driver-top3     & \ms{0.739}{.01} & \ms{0.795}{.01} & \ms{0.824}{.01} & \textcolor{blue}{$+11\%$} & \textcolor{blue}{$+3\%$}  \\
f1    & driver-dnf      & \ms{0.686}{.02} & \ms{0.649}{.01} & \ms{0.691}{.01} & \textcolor{blue}{$+1\%$}  & \textcolor{blue}{$+6\%$}  \\
arxiv & paper-citation  & \ms{0.712}{.04} & \ms{0.807}{.01} & \ms{0.818}{.01} & \textcolor{blue}{$+14\%$} & \textcolor{blue}{$+1\%$}  \\
stack & user-engagement & \ms{0.634}{.03} & \ms{0.835}{.02} & \ms{0.894}{.02} & \textcolor{blue}{$+41\%$} & \textcolor{blue}{$+7\%$}  \\
stack & user-badge      & \ms{0.634}{.02} & \ms{0.668}{.02} & \ms{0.811}{.01} & \textcolor{blue}{$+27\%$} & \textcolor{blue}{$+21\%$} \\
\bottomrule
\end{tabular}
\end{table}

\paragraph{Takeaway.}
Taken together with \texttt{relbench-trial} and the TabArena results, these experiments complete the empirical picture. On i.i.d.\ benchmarks where targets are extension-invariant, structural augmentation brings no consistent gain. On relational tasks where the target depends on inter-row relationships, the hybrid consistently improves over both purely tabular and purely graph-based models --- even on \texttt{driver-dnf}, where neither baseline alone is dominant. The consistent hybrid gains across tasks and datasets support the framework's core claim: tabular and graph-based representations are complementary, and combining them is robust to which regime dominates.

%% file: tables/tab_proportions-tasks-datssets.tex
\begin{table}[h]
\centering
\caption{Proportion of presence of label tasks across train, validation, and test splits for each dataset}
\begin{tabular}{l l c c c}
\hline
Dataset & Task & Train (\%) & Val (\%) & Test (\%) \\
\hline
Synthetic Data & \textsc{UNIQUE} &  13.1& 6.6& 7.5\\
Synthetic Data & \textsc{COUNT} & 24.9 & 18.0& 18.4\\
Synthetic Data & \textsc{COUNT}(=) & 15.4& 15.0 & 14.1 \\
Synthetic Data & \textsc{DOUBLE} & 48.2 & 42.8 & 43.4\\
Synthetic Data & \textsc{DIAMOND} & 30.7 & 28.1 & 24.6\\
Retail Data  & \textsc{UNIQUE} & 0.9& 14.2& 10.5\\
Retail Data  & \textsc{COUNT} & 47.0& 44.6& 55.8\\
Retail Data  & \textsc{COUNT}(=) & 1.4 & 2.7& 1.3\\
Retail Data  & \textsc{DOUBLE} & 41.4 & 42.0& 50.6\\
Retail Data  & \textsc{DIAMOND} & 15.5 & 6.0 & 4.5\\
\hline
\end{tabular}
\label{tab:proportions_tasks}
\end{table}

%% file: tables/tab_tabpfn-ft.tex
\begin{table}[t]
\caption{Effect of finetuning for TabPFN on logical tasks in the synthetic data.}
\label{tab:tabpfn_finetune}
\begin{center}
\begin{small}
\begin{sc}
\begin{tabular}{llrrrr}
\toprule
Task & Metric & Original & Finetuned & Abs.\ diff. & Rel.\ gain/loss \\
\midrule
\textsc{UNIQUE} & roc\_auc & 0.8230 & 0.8744 & +0.0514 & \textcolor{blue}{+6.25\%} \\
\textsc{UNIQUE} & f1      & 0.0250 & 0.0240 & -0.0010 & \textcolor{red}{-4.00\%} \\
\addlinespace
\textsc{COUNT} & roc\_auc & 0.9620 & 0.9290 & -0.0330 & \textcolor{red}{-3.43\%} \\
\textsc{COUNT} & f1      & 0.7750 & 0.8810 & +0.1060 & \textcolor{blue}{+13.68\%} \\
\addlinespace
\textsc{COUNT(=)} & roc\_auc & 0.3400 & 0.3310 & -0.0090 & \textcolor{red}{-2.65\%} \\
\textsc{COUNT(=)} & f1      & 0.0550 & 0.0137 & -0.0413 & \textcolor{red}{-75.09\%} \\
\addlinespace
\textsc{DOUBLE} & roc\_auc & 0.6700 & 0.8260 & +0.1560 & \textcolor{blue}{+23.28\%} \\
\textsc{DOUBLE} & f1      & 0.6990 & 0.7550 & +0.0560 & \textcolor{blue}{+8.01\%} \\
\addlinespace
\textsc{DIAMOND} & roc\_auc & 0.8680 & 0.8671 & -0.0009 & \textcolor{red}{-0.10\%} \\
\textsc{DIAMOND} & f1      & 0.7020 & 0.7258 & +0.0238 & \textcolor{blue}{+3.39\%} \\
\bottomrule
\end{tabular}
\end{sc}
\end{small}
\end{center}
\vskip -0.1in
\end{table}

%% file: tables/tab_synth-training-hyperparameters.tex
\begin{table}[t]
\centering
\caption{Training hyperparameters for synthetic dataset across different tasks.}
\vspace{1em}
\label{tab:hyperparams_synthetic}
\begin{tabular}{lccccc}
\hline
\textbf{Hyperparameter} &
\textbf{Uniqueness} &
\textbf{Count ($>$)} &
\textbf{Count ($=$)} &
\textbf{Double condition} &
\textbf{Diamond} \\
\hline
Layers              & 2   & 2   & 3   & 3   & 4 \\
Hidden dimension    & 64  & 64 & 128 & 64  & 64 \\
Dropout             & 0.0 & 0.0 & 0.0 & 0.0 & 0.0 \\
Learning rate       & 1e-3 & 1e-3 & 1e-3 & 1e-3 & 1e-3 \\
Weight decay        &  1e-4& 1e-4 & 1e-6 & 1e-4 & 1e-4\\
Aggregation              & sum & sum & sum & sum & sum  \\
Epochs              & 75 & 75 & 150 & 75 & 75  \\
\hline
\end{tabular}
\end{table}

%% file: neurips/neurips-tables-figures/table-logical-std.tex
\begin{table}[t]
\centering
\caption{Diagnostic comparison on logical tasks in synthetic and transaction datasets (Part 1: \textsc{Unique}, \textsc{Count}, \textsc{Count}($=$)). F1 of $0$ corresponds to predicting $0$ for all instances. Results averaged over 5 runs (mean ± std).}
\vspace{1em}
\label{tab:logical-tasks-1}
\small
\setlength{\tabcolsep}{4pt}
\renewcommand{\arraystretch}{1.05}
\begin{tabular}{l cc cc cc}
\toprule
\textbf{Model} &
\multicolumn{2}{c}{\textsc{Unique}} &
\multicolumn{2}{c}{\textsc{Count}} &
\multicolumn{2}{c}{\textsc{Count}($=$)} \\
\cmidrule(lr){2-3}\cmidrule(lr){4-5}\cmidrule(lr){6-7}
& AUC & F1 & AUC & F1 & AUC & F1 \\
\midrule
\multicolumn{7}{c}{\emph{Synthetic transactions}} \\
LightGBM   & \ms{0.566}{.012} & \ms{0.123}{.007} & \ms{0.889}{.001} & \ms{0.000}{.000} & \ms{0.675}{.030} & \ms{0.314}{.009} \\
RealMLP    & \ms{0.516}{.012} & \ms{0.123}{.007} & \ms{0.933}{.005} & \ms{0.733}{.002} & \ms{0.521}{.007} & \ms{0.000}{.000} \\
ModernNCA  & \ms{0.597}{.009} & \ms{0.000}{.000} & \ms{0.862}{.017} & \ms{0.569}{.007} & \ms{0.544}{.012} & \ms{0.000}{.000} \\
TabPFN2.5  & \ms{0.823}{.007} & \ms{0.025}{.003} & \ms{0.962}{.008} & \ms{0.775}{.004} & \ms{0.340}{.020} & \ms{0.055}{.004} \\
GNN (Sage) & \textbf{\ms{0.997}{.002}} & \textbf{\ms{0.986}{.003}} & \textbf{\ms{1.000}{.000}} & \textbf{\ms{1.000}{.000}} & \textbf{\ms{0.988}{.005}} & \textbf{\ms{0.859}{.002}} \\
\midrule
\multicolumn{7}{c}{\emph{Real-world retail transactions}} \\
LightGBM   & \ms{0.570}{.027} & \ms{0.190}{.012} & \ms{0.610}{.092} & \ms{0.000}{.000} & \ms{0.520}{.022} & \ms{0.027}{.015} \\
RealMLP    & \ms{0.853}{.007} & \ms{0.190}{.010} & \ms{0.929}{.012} & \ms{0.881}{.017} & \ms{0.477}{.021} & \ms{0.027}{.005} \\
ModernNCA  & \ms{0.749}{.012} & \ms{0.060}{.007} & \ms{0.968}{.010} & \ms{0.522}{.004} & \ms{0.438}{.018} & \ms{0.000}{.000} \\
TabPFN2.5  & \ms{0.793}{.009} & \ms{0.404}{.004} & \ms{0.827}{.008} & \ms{0.001}{.000} & \ms{0.758}{.017} & \ms{0.023}{.005} \\
GNN (Sage) & \textbf{\ms{0.998}{.001}} & \textbf{\ms{0.988}{.001}} & \textbf{\ms{0.999}{.001}} & \textbf{\ms{0.967}{.005}} & \textbf{\ms{0.903}{.062}} & \textbf{\ms{0.171}{.055}} \\
\bottomrule
\end{tabular}
\end{table}

\begin{table}[t]
\centering
\caption{Diagnostic comparison on logical tasks in synthetic and transaction datasets (Part 2: \textsc{Double}, \textsc{Diamond}). F1 of $0$ corresponds to predicting $0$ for all instances. Results averaged over 5 runs (mean ± std).}
\vspace{5pt}
\label{tab:logical-tasks-2}
\small
\setlength{\tabcolsep}{4pt}
\renewcommand{\arraystretch}{1.05}
\begin{tabular}{l cc cc}
\toprule
\textbf{Model} &
\multicolumn{2}{c}{\textsc{Double}} &
\multicolumn{2}{c}{\textsc{Diamond}} \\
\cmidrule(lr){2-3}\cmidrule(lr){4-5}
& AUC & F1 & AUC & F1 \\
\midrule
\multicolumn{5}{c}{\emph{Synthetic transactions}} \\
LightGBM   & \ms{0.789}{.007} & \ms{0.729}{.012} & \ms{0.846}{.021} & \ms{0.682}{.009} \\
RealMLP    & \ms{0.785}{.012} & \ms{0.686}{.030} & \ms{0.796}{.021} & \ms{0.603}{.017} \\
ModernNCA  & \ms{0.720}{.033} & \ms{0.670}{.017} & \ms{0.798}{.009} & \ms{0.593}{.005} \\
TabPFN2.5  & \ms{0.670}{.025} & \ms{0.699}{.013} & \ms{0.868}{.030} & \textbf{\ms{0.702}{.018}} \\
GNN (Sage) & \textbf{\ms{0.887}{.012}} & \textbf{\ms{0.777}{.007}} & \textbf{\ms{0.899}{.012}} & \ms{0.676}{.028} \\
\midrule
\multicolumn{5}{c}{\emph{Real-world retail transactions}} \\
LightGBM   & \ms{0.847}{.039} & \ms{0.815}{.030} & \textbf{\ms{0.821}{.010}} & \ms{0.199}{.012} \\
RealMLP    & \ms{0.703}{.015} & \ms{0.000}{.000} & \ms{0.792}{.007} & \ms{0.139}{.003} \\
ModernNCA  & \ms{0.910}{.021} & \ms{0.857}{.005} & \ms{0.725}{.022} & \ms{0.104}{.012} \\
TabPFN2.5  & \ms{0.463}{.072} & \ms{0.003}{.003} & \ms{0.659}{.034} & \ms{0.007}{.005} \\
GNN (Sage) & \textbf{\ms{0.910}{.012}} & \textbf{\ms{0.863}{.007}} & \textbf{\ms{0.807}{.012}} & \textbf{\ms{0.251}{.007}} \\
\bottomrule
\end{tabular}
\end{table}

%% file: tables/tab_retail-full-vs-clean.tex
\begin{table}[t]
\centering
\caption{Performance on original vs.\ cleaned retail data (Unique, Count).
$\Delta = \text{Clean} - \text{Orig}$.}
\vspace{1em}
\label{tab:appendix_retail_part1}
\begin{tabular}{lccccccccc}
\toprule
\textbf{Model} &
\multicolumn{3}{c}{\textbf{$t_1$: unique}} &
\multicolumn{3}{c}{\textbf{$t_2$: count($>$)}} &
\multicolumn{3}{c}{\textbf{$t_3$: count($=$)}} \\
\cmidrule(lr){2-4}
\cmidrule(lr){5-7}
\cmidrule(lr){8-10}
& \textbf{Orig} & \textbf{Clean} & $\boldsymbol{\Delta}$
& \textbf{Orig} & \textbf{Clean} & $\boldsymbol{\Delta}$
& \textbf{Orig} & \textbf{Clean} & $\boldsymbol{\Delta}$ \\
\midrule
\multicolumn{10}{c}{\textit{F1-score}} \\
\midrule
GNN (SAGE)
& 0.988 & 0.989 & +0.001
& 0.967 & 0.997 & +0.030
& 0.171 & 1.000 & \bf+0.829 \\
LightGBM
& 0.190 & 0.190 & +0.000
& 0.000 & 0.000 & +0.000
& 0.027 & 0.027 & +0.000 \\
RealMLP
& 0.190 & 0.191 & +0.001
& 0.881 & 0.885 & +0.004
& 0.027 & 0.027 & +0.000 \\
TabPFN
& 0.404 & 0.406 & +0.002
& 0.001 & 0.001 & -0.000
& 0.023 & 0.019 & -0.004 \\
\midrule
\multicolumn{10}{c}{\textit{ROC--AUC}} \\
\midrule
GNN (SAGE)
& 0.998 & 0.999 & +0.001
& 0.999 & 0.999 & +0.000
& 0.903 & 1.000 & +0.097 \\
LightGBM
& 0.570 & 0.540 & -0.030
& 0.610 & 0.610 & +0.000
& 0.520 & 0.520 & +0.000 \\
RealMLP
& 0.853 & 0.708 & -0.145
& 0.929 & 0.882 & -0.047
& 0.477 & 0.570 & +0.093 \\
TabPFN
& 0.793 & 0.802 & +0.009
& 0.827 & 0.561 & -0.266
& 0.758 & 0.716 & -0.042 \\
\bottomrule
\end{tabular}
\end{table}

\begin{table}[t]
\centering
\caption{Performance on original vs.\ cleaned retail data (Double, Diamond).
$\Delta = \text{Clean} - \text{Orig}$.}
\vspace{1em}
\label{tab:appendix_retail_part2}
\begin{tabular}{lcccccc}
\toprule
\textbf{Model} &
\multicolumn{3}{c}{\textbf{$t_3$: double condition}} &
\multicolumn{3}{c}{\textbf{$t_4$: diamond}} \\
\cmidrule(lr){2-4}
\cmidrule(lr){5-7}
& \textbf{Orig} & \textbf{Clean} & $\boldsymbol{\Delta}$
& \textbf{Orig} & \textbf{Clean} & $\boldsymbol{\Delta}$ \\
\midrule
\multicolumn{7}{c}{\textit{F1-score}} \\
\midrule
GNN (SAGE)
& 0.863 & 0.855 & -0.008
& 0.251 & 0.249 & -0.002 \\
LightGBM
& 0.815 & 0.672 & -0.143
& 0.199 & 0.087 & -0.112 \\
RealMLP
& 0.000 & 0.819 & +0.819
& 0.139 & 0.169 & +0.030 \\
TabPFN
& 0.003 & 0.003 & +0.000
& 0.007 & 0.007 & +0.000 \\
\midrule
\multicolumn{7}{c}{\textit{ROC--AUC}} \\
\midrule
GNN (SAGE)
& 0.910 & 0.908 & -0.002
& 0.807 & 0.836 & +0.029 \\
LightGBM
& 0.847 & 0.420 & -0.427
& 0.821 & 0.769 & -0.052 \\
RealMLP
& 0.703 & 0.827 & +0.124
& 0.792 & 0.743 & -0.049 \\
TabPFN
& 0.463 &  0.380 & -0.083
& 0.659 & 0.560 & -0.099 \\
\bottomrule
\end{tabular}
\end{table}